\documentclass[11pt]{article}
\usepackage{amssymb}
\usepackage{amsfonts}
\usepackage{amsmath}
\usepackage{amsthm}
\usepackage{latexsym}

\usepackage{bbm}

\usepackage{enumitem}

\usepackage{mathtools}

\usepackage{xspace}

\usepackage[pagebackref=true]{hyperref}
\hypersetup{
    unicode=false,          
    colorlinks=true,        
    linkcolor=red,          
    citecolor=blue,        
    filecolor=magenta,      
    urlcolor=cyan           
}

\usepackage{hyperref,cleveref}
\usepackage{amsmath}
\usepackage{amsthm}
\usepackage{amsfonts}
\usepackage{fullpage,appendix}

\usepackage[ruled]{algorithm}
\usepackage{algpseudocode}
\usepackage{algorithmicx}

\usepackage{color}
\usepackage{tikz}

\providecommand{\remove}[1]{}
\providecommand{\eg}{{\em e.g.}~}

\newcommand{\etal}{et al.\ }

\newcommand{\ignore}[1]{}
\definecolor{corlinks}{RGB}{64,128,128}
\definecolor{cormenu}{RGB}{0,37,94}
\definecolor{corurl}{RGB}{0,46,91}
\definecolor{darkgreen}{rgb}{0,0.5,0}

\newcommand{\err}{\mathsf{err}}

\newcommand{\CSDR}{\mathsf{cRSD}}
\newcommand{\RSD}{\mathsf{RSD}}
\newcommand{\SD}{\mathsf{SD}}

\newcommand{\fract}{\mathsf{frac}}

\newcommand{\Hyp}{\mathsf{Hyp}}

\newcommand{\LR}{\mathsf{LR}}

\newcommand{\rk}{\mathsf{rk}}

\newcommand{\E} {\mathbb{E}}

\newcommand{\R}{\mathbb{R}}

\newcommand{\D}{\mathcal{D}}
\newcommand{\calD}{\mathcal{D}}
\newcommand{\calP}{\mathcal{P}}
\newcommand{\calB}{\mathcal{B}}

\newcommand{\calF}{\mathcal{F}}
\newcommand{\calZ}{\mathcal{Z}}

\newcommand{\calA}{\mathcal{A}}
\newcommand{\calC}{\mathcal{C}}

\newcommand{\A}{\mathcal{A}}

\newcommand{\STAT}{\mathsf{STAT}}
\newcommand{\VSTAT}{\mathsf{VSTAT}}

\newcommand{\F}{\mathbb{F}}

\newcommand{\zo}{\{0, 1\}}

\newcommand{\KL}{{\mathrm{KL}}}
\newcommand{\Div}{{\mathrm{D}}}
\newcommand{\KLR}{R_{\mathrm{KL}}}

\newcommand{\Ex}{\mathbb E}

\renewcommand{\cal}[1]{\mathcal{#1}}
\newcommand{\eps}{\epsilon}

\newcommand{\cald}{\mathcal{D}}

\newtheorem{fact}{Fact}[section]
\newtheorem{definition}[fact]{Definition}
\newtheorem{defn}[fact]{Definition}

\newtheorem{theorem}[fact]{Theorem}
\newtheorem{lemma}[fact]{Lemma}
\newtheorem{lem}[fact]{Lemma}
\newtheorem{corollary}[fact]{Corollary}

\newtheorem{proposition}[fact]{Proposition}

\newtheorem{remark}[fact]{Remark}

\makeatletter
\newtheorem*{rep@theorem}{\rep@title}
\newcommand{\newreptheorem}[2]{%
\newenvironment{rep#1}[1]{%
 \def\rep@title{#2 \ref{##1}}%
 \begin{rep@theorem}}%
 {\end{rep@theorem}}}
\makeatother

\newreptheorem{theorem}{Theorem}

\newcommand{\COMM}{\mbox{BS}}

\newif\ifnotes\notestrue

\ifnotes

\newcommand{\bnote}[1]{\textcolor{red}{{\bf (Badih} {#1}{\bf ) }} \marginpar{\tiny\bf
             \begin{minipage}[t]{0.5in}
               \raggedright Badih
            \end{minipage}}}
\newcommand{\gnote}[1]{\textcolor{red}{{\bf (GV:} {#1}{\bf ) }} \marginpar{\tiny\bf
             \begin{minipage}[t]{0.5in}
               \raggedright GV
                \end{minipage}}}
\else
\newcommand{\enote}[1]{}
\newcommand{\bnote}[1]{}
\newcommand{\gnote}[1]{}
\fi

\usepackage{empheq}

\title{On the Power of Learning from $k$-Wise Queries}
\author{
Vitaly Feldman \\
IBM Research - Almaden
\and
Badih Ghazi\thanks{Work done while at IBM Research - Almaden.}\\
Computer Science and Artificial Intelligence Laboratory, MIT
}

\date{\today}

\begin{document}

\newpage

\maketitle

\abstract{
Several well-studied models of access to data samples, including statistical queries, local differential privacy and low-communication algorithms rely on queries that provide information about a function of a single sample. (For example, a statistical query (SQ) gives an estimate of $\E_{x\sim D}[q(x)]$ for any choice of the query function $q:X\rightarrow \R$, where $D$ is an unknown data distribution.) Yet some data analysis algorithms rely on properties of functions that depend on multiple samples. Such algorithms would be naturally implemented using $k$-wise queries each of which is specified by a function $q:X^k\rightarrow \R$. Hence it is natural to ask whether algorithms using $k$-wise queries can solve learning problems more efficiently and by how much.

Blum, Kalai, Wasserman~\cite{blum2003noise} showed that for any weak PAC learning problem over a fixed distribution, the complexity of learning with $k$-wise SQs is smaller than the (unary) SQ complexity by a factor of at most $2^k$. We show that for more general problems over distributions the picture is substantially richer. For every $k$, the complexity of distribution-independent PAC learning with $k$-wise queries can be exponentially larger than learning with $(k+1)$-wise queries. We then give two approaches for simulating a $k$-wise query using unary queries. The first approach exploits the structure of the problem that needs to be solved. It generalizes and strengthens (exponentially) the results of Blum \etal \cite{blum2003noise}. It allows us to derive strong lower bounds for learning DNF formulas and stochastic constraint satisfaction problems that hold against algorithms using $k$-wise queries. The second approach exploits the $k$-party communication complexity of the $k$-wise query function. 
}

\newpage

\tableofcontents

\newpage

\section{Introduction}

In this paper, we consider several well-studied models of learning from i.i.d.~samples that restrict the algorithm's access to samples to evaluation of functions of an individual sample. The primary model of interest is the statistical query model introduced by Kearns \cite{kearns1998efficient} as a restriction of Valiant's PAC learning model \cite{valiant1984theory}. The SQ model allows the learning algorithm to access the data only via \emph{statistical queries}, which are estimates of the expectation of any function of labeled examples with respect to the input distribution $D$. More precisely, if the domain of the functions is $Z$, then a statistical query is specified by a function $\phi:Z \times \{\pm 1\} \to [-1,1]$ and by a tolerance parameter $\tau$. Given $\phi$ and $\tau$, the statistical query oracle returns a value $v$ which satisfies $|v - \Ex_{(z,b) \sim D}[\phi(z,b)]| \le \tau$.

The SQ model is known to be closely-related to several other models and concepts: linear statistical functionals \cite{wasserman2013all}, learning with a distance oracle \cite{Ben-DavidIK90}, approximate counting (or linear) queries extensively studied in differential privacy (e.g., \cite{DinurN03,BlumDMN05,DworkMNS06,roth2010interactive}), local differential privacy \cite{kasiviswanathan2011can}, evolvability \cite{valiant2009evolvability,feldman2008evolvability}, and algorithms that extract a small amount of information from each sample \cite{Ben-DavidD98,FeldmanGRVX:12,FeldmanPV:13,SteinhardtVW16}. This allows to easily extend the discussion in the context of the SQ model to these related models and we will formally state several such corollaries.

Most standard algorithmic approaches used in learning theory are known to be implementable using SQs (e.g., \cite{blum1998polynomial,DunaganV04,BlumDMN05,chu2007map,FeldmanPV:13,BalcanF15,FeldmanGV:15}) leading to numerous  theoretical (e.g., \cite{BalcanBFM12,de2015learning,dwork2015preserving}) and practical (e.g., \cite{chu2007map,roy2010airavat,sujeeth2011optiml,dwork2015generalization}) applications. SQ algorithms have also been recently studied outside the context of learning theory \cite{FeldmanGRVX:12,FeldmanPV:13,FeldmanGV:15}. In this case we denote the domain of data samples by $X$.

Another reason for the study of SQ algorithms is that it is possible to prove information-theoretic lower bounds on the complexity of any SQ algorithm that solves a given problem. Given that a large number of algorithmic approaches to problems defined over data sampled i.i.d.~from some distribution can be implemented using statistical queries, this provides a strong and unconditional evidence of the problem's hardness. For a number of central problems in learning theory and complexity theory, unconditional lower bounds for SQ algorithms are known that closely match the known {\em computational} complexity upper bounds for those problems (\eg \cite{BlumFJ+:94, FeldmanGRVX:12,FeldmanPV:13,DachmanFTWW:15,DiakonikolasKS:16}). 

A natural strengthening of the SQ model (and other related models) is to allow function over $k$-tuples of samples instead of a single sample. That is, for a $k$-ary query function $\phi:X^k \to [-1,1]$, the algorithm can obtain an estimate of $\Ex_{x_1,\ldots, x_k \sim D}[\phi(x_1,\ldots,x_k)]$. It can be seen as interpolating between the power of algorithms that can see all the samples at once and those that process a single sample at a time. While most algorithms can be implemented using standard unary queries, some algorithms are known to require such more powerful queries. The most well-known example is Gaussian elimination over $\mathbb{F}_2^n$ that is used for learning parity functions. Standard hardness amplification techniques rely on mapping examples of a function $f(z)$ to examples of a function $g(f(z_1),\ldots,f(z_k))$ (for example \cite{BonehLipton:93,FeldmanLS:11colt}). Implementing such reduction requires $k$-wise queries and, consequently, to obtain a lower bound for solving an amplified problem with unary queries one needs a lower bound against solving the original problem with $k$-wise queries. A simple example of 2-wise statistical query is collision probability $\Pr_{x_1,x_2 \sim D} [x_1=x_2]$ that is used in several distribution property testing algorithms.

\subsection{Previous work}
Blum, Kalai and Wasserman \cite{blum2003noise} introduced and studied the power of $k$-wise SQs in the context of weak {\em distribution-specific} PAC learning: that is the learning algorithm observes pairs $(z,b)$, where $z$ is chosen randomly from some fixed and known distribution $P$ over $Z$ and $b=f(z)$ for some unknown function $f$ from a class of functions $\cal C$. They showed that if a class of functions $\cal C$ can be learned with error $1/2-\lambda$ relative to distribution $P$ using $q$ $k$-wise SQs of tolerance $\tau$ then it can be learned with error $\max\{1/2 - \lambda, 1/2 -\tau/2^k \}$ using $O(q \cdot 2^k)$ unary SQs of tolerance $\tau/2^k$.

More recently, Steinhardt \etal \cite{SteinhardtVW16} considered $k$-wise queries in the $b$-bit sampling model in which for any query function $\phi:X^k \to \zo^{b}$ an algorithm get the value $\phi(x_1,\ldots,x_k)$ for $x_1,\ldots,x_k$ drawn randomly and independently from $D$ (it is referred to as one-way communication model in their work). They give a general technique for proving lower bounds on the number of such queries that are required to solve a given problem.

\subsection{Our results}
In this work, we study the relationship between the power of $k$-wise queries and unary queries for arbitrary problems in which the input is determined by some unknown input distribution $D$ that belongs a (known) family of distributions $\D$ over domain $X$. 

\paragraph{Separation for distribution-independent learning:}
We first demonstrate that for distribution-independent PAC learning $(k+1)$-wise queries are exponentially stronger than $k$-wise queries. We say that the $k$-wise SQ complexity of a certain problem is $m$ if $m$ is the smallest such that there exists an algorithm that solves the problem using $m$ $k$-wise SQs of tolerance $1/m$.
\begin{theorem}\label{thm:k_wise_sep} (Informal)
For every positive integer $k$ and any prime number $p$, there is a concept class $\calC$ of Boolean functions defined over a domain of size $p^{k+1}$ such that the $(k+1)$-wise SQ complexity of distribution-independent PAC learning $\calC$ with is $O_k(\log{p})$ whereas the $k$-wise SQ complexity of distribution-independent PAC learning of $\calC$ is $\Omega_k(p^{1/4})$.
\end{theorem}

The class of functions we use consists of all indicator functions of $k$-dimensional affine subspaces of $\F_p^{k+1}$. Our lower bound is a generalization of the lower bound for unary SQs in \cite{Feldman:16sqd} (that corresponds to $k=1$ case of the lower bound). A simple but important observation that allows us to easily adapt the techniques from earlier works on SQs to the $k$-wise case is that a $k$-wise SQ for an input distribution $D \in \D$ are equivalent to unary SQ for a product distribution $D^k$.

The upper bound relies on the ability to find the affine subspace given $k+1$ positively labeled and linearly independent points in $\F_p^{k+1}$.  Unfortunately, for general distributions the probability of observing such a set of points can be arbitrarily small. Nevertheless, we argue that there will exist a unique lower-dimensional affine subspace that contains enough probability mass of all the positive points in this case. This upper bound essentially implies that given $k$-wise queries one can solve problems that require Gaussian elimination over a system of $k$ equations.

\paragraph{Reduction for flat $\D$:}\label{subsec:intro_flat}
The separation in Theorem~\ref{thm:k_wise_sep} relies on using an unrestricted class of distributions $\D$. We now prove that if $\D$ is ``flat" relative to some ``central" distribution $\bar{D}$ then one can upper bound the power of $k$-wise queries in terms of unary queries.


\begin{defn}[Flat class of distributions]
Let $\calD$ be a set of distributions over $X$, and $\bar{D}$ a distribution over $X$. For $\gamma \geq 1$ we say that $\calD$ is $\gamma$-flat if there exists some distribution $\bar{D}$ over $X$ such that for all $D \in \mathcal{D}$ and all measurable subsets $E \subseteq X$, we have that $\Pr_{x\sim D}[x \in E] \le \gamma \cdot \Pr_{x\sim \bar{D}}[x \in E]$.
\end{defn}

We now state our upper bound for flat classes of distributions, where we use $\STAT_D^{(k)}(\tau)$ to refer to the oracle that answers $k$-wise SQs for $D$ with tolerance $\tau$.
\begin{theorem}\label{thm:flat}
Let $\gamma \geq 1$, $\tau > 0$ and $k$ be any positive integer. Let $X$ be a domain and $\calD$ a $\gamma$-flat class of distributions over $X$. There exists a randomized algorithm that given any $\delta > 0$ and a $k$-ary function $\phi: X^k \to [-1,1]$ estimates $D^k[\phi]$ within $\tau$  for every (unknown) $D \in \calD$ with success probability at least $1-\delta$ using $$\tilde{O}\bigg( \frac{\gamma^{k-1} \cdot k^3}{\tau^3} \cdot \log (1/\delta)\bigg)$$
queries to $\STAT_D^{(1)}(\tau/(6 \cdot k))$.\end{theorem}

To prove this result, we use a recent general characterization of SQ complexity \cite{Feldman:16sqd}. This characterization reduces the problem of estimating $D^k[\phi]$ to the problem of distinguishing between $D^k$ and $D_1^k$ for every $D \in \D$ and some fixed $D_1$. We show that when solving this problem, any $k$-wise query can be replaced by a randomly chosen set of unary queries. Finding these queries requires drawing samples from $D^{k-1}$. As we do not know $D$, we use $\bar{D}$ instead incurring the $\gamma^{k-1}$ overhead in sampling. In Section \ref{sec:pf_flat} we show that weaker notions of ``flatness" based on different notions of divergence between distributions can also be used in this reduction.

It is easy to see that, when PAC learning $\calC$ with respect to a fixed distribution $P$ over $Z$, the set of input distributions is 2-flat (relative to the distribution that is equal to $P$ on $Z$ and gives equal weight $1/2$ to each label). Therefore, our result generalizes the results in \cite{blum2003noise}. More importantly, the tolerance in our upper bound scales linearly with $k$ rather than exponentially (namely, $\tau/2^k)$.

This result can be used to obtain lower bounds against $k$-wise SQs algorithms from lower bounds against unary SQ algorithms. In particular, it can be used to rule out reductions that require looking at $k$ points of the original problem instance to obtain each point of the new problem instance. As an application, we obtain exponential lower bounds for solving constraint stochastic satisfaction problems and DNF learning by $k$-wise SQ algorithm with $k=n^{1-\alpha}$ for any constant $\alpha > 0$ from lower bounds for CSPs given in \cite{FeldmanPV:13}.
We state the result for learning DNF here. Definitions and the lower bound for CSPs can be found in Section \ref{sec:lower-bounds}.
\begin{theorem}\label{thm:dnf-k-wise-intro}
For any constant $\alpha >0$ (independent of $n$), there exists a constant $\beta>0$ such that
 any  algorithm that learns DNF formulas of size $n$ with error $<1/2 - n^{- \beta \log n}$ and success probability at least $2/3$ requires at least $2^{n^{1-\alpha}}$ calls to $\STAT^{(n^{1-\alpha})}_D(n^{- \beta \log n})$.
\end{theorem}
This lower bound is based on a simple and direct reduction from solving the stochastic CSP that arises in Goldreich's proposed PRG \cite{goldreich2000candidate} to learning DNF that is of independent interest (see Lemma \ref{lem:reduce-dnf}). For comparison, the standard SQ lower bound for learning polynomial size DNF \cite{BlumFJ+:94} relies on hardness of learning parities of size $\log n$ over the uniform distribution. Yet, parities of size $\log n$ can be easily learned from $(\log^2 n)$-wise statistical queries (since solving a system of $\log^2 n$ linear equations will uniquely identify a $\log n$-sparse parity function). Hence our lower bound holds against qualitatively stronger algorithms.  Our lower bound is also exponential in the number of queries whereas the known argument implies only a quasipolynomial lower bound\footnote{We remark that an exponential lower bound on the number of queries has not been previously stated even for unary SQs. The unary version can be derived from known results as explained in Section \ref{sec:lower-bounds}.}.

\paragraph{Reduction for low-communication queries:}
Finally, we point out that $k$-wise queries that require little information about each of the inputs can also be simulated using unary queries. This result is a simple corollary of the recent work of Steinhardt \etal \cite{SteinhardtVW16} who show that any computation that extracts at most $b$ bits from each of the samples (not necessarily at once) can be simulated using unary SQs.

\begin{theorem}\label{thm:sq_and_cc}
Let $\phi:X^k \to \{\pm 1\}$ be a function, and assume that $\phi$ has $k$-party public-coin randomized communication complexity of $b$ bits per party with success probability $2/3$. Then, there exists a randomized algorithm that, with probability at least $1-\delta$, estimates $\Ex_{x \sim D^k}[\phi(x)]$ within $\tau$ using $O(b \cdot k \cdot \log(1/\delta)/\tau^2)$ queries to $\STAT^{(1)}_D(\tau')$ for some $\tau' = \tau^{O(b)}/k$.
\end{theorem}

As a simple application of Theorem~\ref{thm:sq_and_cc}, we show a unary SQ algorithm that estimates the collision probability of an unknown distribution $D$ within $\tau$ using $1/\tau^2$ queries $\STAT^{(1)}_{D}(\tau^{O(1)})$. The details appear in Section~\ref{sec:sq_cc}.

\paragraph{Corollaries for related models:}
Our separation result and reductions imply similar results for $k$-wise versions of two well-studied learning models: local differential privacy and the $b$-bit sampling model. 

Local differentially private algorithms \cite{kasiviswanathan2011can} (also referred to as randomized response) are differentially private algorithms in which each sample goes through a differentially private transformation chosen by the analyst. This model is the focus of recent privacy preserving industrial applications by Google \cite{ErlingssonPK14} and Apple. We define a $k$-wise version of this model in which analyst's differentially private transformations are applied to $k$-tuples of samples. This model interpolates naturally between the usual (or global) differential privacy and the local model.

Kasiviswanathan \etal\cite{kasiviswanathan2011can} showed that a concept class is learnable by a local differentially private algorithm if and only if it is learnable in the SQ model. Hence up to polynomial factors the models are equivalent (naturally, such polynomial factors are important for applications but here we focus only on the high-level relationships between the models). This result also implies that $k$-local differentially private algorithms (formally defined in Section~\ref{subsec:local_DP}) are equivalent to $k$-wise SQ algorithms (up to a polynomial blow-up in the complexity). Theorem~\ref{thm:k_wise_sep} then implies an exponential separation between $k$-wise and $(k+1)$-wise local differentially private algorithms (see Corollary~\ref{cor:local_DP} for details). It can be seen as a substantial strengthening of a separation between the local model and the global one also given in \cite{kasiviswanathan2011can}. The reductions in Theorem \ref{thm:flat} and Theorem \ref{thm:sq_and_cc} imply two approaches for simulating  $k$-local differentially private algorithms using 1-local algorithms.

The SQ model is also known to be equivalent (up to a factor polynomial in $2^b$) to the $b$-bit sampling model introduced by Ben-David and Dichterman \cite{Ben-DavidD98} and studied more recently in \cite{FeldmanGRVX:12,FeldmanPV:13,ZhangDJW13,SteinhardtD15,SteinhardtVW16}. Lower bounds for the $k$-wise version of this model are given in \cite{ZhangDJW13,SteinhardtVW16}.
Our results can be easily translated to this model as well. We provide additional details in Section~\ref{sec:apps}.

\section{Preliminaries}
For any distribution $D$ over a domain $X$ and any positive integer $k$, we denote by $D^k$ the distribution over $X^k$ obtained by drawing $k$ i.i.d. samples from $D$. For a distribution $D$ over a domain $X$ and a function $\phi: X \to \mathbb{R}$, we denote $D[\phi] \doteq \Ex_{x \sim D}[\phi(x)]$.


Next, we formally define the $k$-wise SQ oracle.

\begin{defn}\label{def:kw_sq}
Let $D$ be a distribution over a domain $X$ and $\tau > 0$. A $k$-wise statistical query oracle $\STAT^{(k)}_D(\tau)$ is an oracle that given as input any function $\phi:X^k \to [-1,+1]$, returns some value $v$ such that $|v-\Ex_{x \sim D^k}[\phi(x)]| \le \tau$.
\end{defn}

We say that a $k$-wise SQ algorithm is given access to $\STAT^{(k)}(\tau)$, if for every when the algorithm is given access to  $\STAT^{(k)}_D(\tau)$, where $D$ is the input distribution. We note that for $k = 1$, Definition~\ref{def:kw_sq} reduces to the usual definition of an SQ oracle that was first introduced by Kearns \cite{kearns1998efficient}. 
The $k$-wise SQ complexity of solving a problem with access to $\STAT^{(k)}(\tau)$  is the minimum number of queries $q$ for which exists a $k$-wise SQ algorithm with access to $\STAT^{(k)}(\tau)$ that solves the problem using at most $q$ queries. Our discussion and results can also be easily extended to the stronger $\VSTAT$ oracle defined in \cite{FeldmanGRVX:12} and to more general real-valued queries using the reductions in \cite{Feldman:16sqvar}.

The PAC learning \cite{valiant1984theory} is defined as follows.
\begin{defn}
For a class $\calC$ of Boolean-valued functions over a domain $Z$, a PAC learning algorithm for $\calC$ is an algorithm that for every $P$ distribution over $Z$ and $f \in \calC$, given an error parameter $\epsilon > 0$, failure probability $\delta > 0$ and access to i.i.d.~labeled examples of the form $(x,f(x))$ where $x \sim P$, outputs a hypothesis function $h$ that, with probability at least $1-\delta$, satisfies $\Pr_{x \sim P}[h(x) \neq f(x)] \le \epsilon$.
\end{defn}

We next define one-vs-many decision problems, which will be used in the proofs in our Section~\ref{sec:pf_sep} and Section~\ref{sec:pf_flat}.
\begin{defn}[Decision problem $\calB(\calD, D_0)$]
Let $\calD$ be a set of distributions and $D_0$ a reference distribution over a set $X$. We denote by $\calB(\calD, D_0)$ the decision problem where we are given access to a distribution $D \in \calD \cup \{D_0\}$ and wish to distinguish whether $D \in \calD$ or $D = D_0$.
\end{defn}

\section{Separation of $(k+1)$-wise from $k$-wise queries}\label{sec:pf_sep}
\newcommand{\ind}{\mathbbm{1}}

We start by describing the concept class $\calC$ that we use to prove Theorem~\ref{thm:k_wise_sep}. Let $\ell$ and $k$ be positive integers with $\ell \geq k+1$. The domain will be $\mathbb{F}_p^{\ell}$. For every $a = (a_1,\dots, a_{\ell}) \in \mathbb{F}_p^{\ell}$, we consider the hyperplane
$$ \Hyp_a \doteq \{ z = (z_1,\dots,z_{\ell}) \in \mathbb{F}_p^{\ell}: z_{\ell} = a_1 z_1 + \dots + a_{\ell-1} z_{\ell-1} + a_{\ell}\}.$$
We then define the Boolean-valued function $f_a: \mathbb{F}_p^{\ell} \to \{\pm 1\}$ to be the indicator function of the subset $\Hyp_a \subseteq \mathbb{F}_p^{\ell}$, i.e., for every $z \in \mathbb{F}_p^{\ell}$,

\[
 f_a(z) = \begin{dcases*}
        +1  & if $z \in \Hyp_a$,\\
        -1 & otherwise.
        \end{dcases*}
\]

Then, we will consider the concept classes $\calC_{\ell} \doteq \{f_a: a \in \mathbb{F}_p^{\ell}\}$. We denote $\calC \doteq \calC_{k+1}$. We start by stating our upper bound on the $(k+1)$-wise SQ complexity of the distribution-independent PAC learning of $\calC_{k+1}$.

\begin{lem}[$(k+1)$-wise upper bound]\label{le:k_ub}
Let $p$ be a prime number and $k$ be a positive integer. There exists a distribution-independent PAC learning algorithm for $\calC_{k+1}$ that makes at most $t \cdot \log(1/\eps) $ queries to $\STAT^{(k+1)}(\eps/t)$, for some $t = O_k(\log{p})$.
\end{lem}

We next state our lower bound on the $k$-wise SQ complexity of the same tasks considered in Lemma~\ref{le:k_ub}.

\begin{lem}[$k$-wise lower bound]\label{le:k_plus_1_lb}
 Let $p$ be a prime number and $\ell$, $k$ be positive integers with $\ell \geq k+1$ and $k = O(p)$. There exists $t =  \Omega\big(p^{(\ell-k)/4}\big)$ such that any distribution-independent PAC learning alogrithm for $\calC_{\ell}$ with error at most $1/2-2/t$ that is given access to $\STAT^{(k)}(1/t)$ needs at least $t$ queries.
\end{lem}

Note that  Lemma~\ref{le:k_ub} and Lemma~\ref{le:k_plus_1_lb} imply Theorem~\ref{thm:k_wise_sep}.

\subsection{Upper bound}


\paragraph{Notation}
We first introduce some notation that will be useful in the description of our algorithm. For any matrix $M$ with entries in the finite field $\mathbb{F}_p$, we denote by $\rk(M)$ the rank of $M$ over $\mathbb{F}_p$. Let $(a_1,\dots,a_{k+1}) \in \mathbb{F}_p^{k+1}$ be the unknown vector that defines $f_a$ and $P$ be the unknown distribution over tuples $(z_1, \dots, z_{k+1}) \in \mathbb{F}_p^{k+1}$. 

Note that $\Hyp_a$ is an affine subspace of $\mathbb{F}_p^{k+1}$. To simplify our treatment of affine subspaces, we embed the points of $\mathbb{F}_p^{k+1}$ into $\mathbb{F}_p^{k+2}$ by mapping each $z \in \mathbb{F}_p^{k+1}$ to $(z,1)$. This embedding maps every affine subspace $V$ of $\mathbb{F}_p^{k+1}$ to a linear subspace $W$ of $\mathbb{F}_p^{k+2}$, namely the span of the image of $V$ under our embedding. Note that this mapping is one-to-one and allows us to easily recover $V$ from $W$ as $V = \{z \in \mathbb{F}_p^{k+1} \ | \ (z,1) \in W\}$.  Hence given $k+1$ examples
$$\big((z_{1,1}, \dots, z_{1,k+1}),b_1\big),\big((z_{2,1}, \dots, z_{2,k+1}),b_2\big), \dots, \big((z_{k+1,1}, \dots, z_{k+1,k+1}),b_{k+1}\big)$$  we define the matrix:
\begin{equation}\label{eq:Z_mat_def}
Z \doteq
\begin{bmatrix}
z_{1,1}       & z_{1,2} &  \cdot & z_{1,k+1} & 1 \\
z_{2,1}       & z_{2,2} & \cdot & z_{2,k+1} & 1 \\
\cdot       & \cdot & \cdot & \cdot & \cdot \\
\cdot       & \cdot & \cdot & \cdot & \cdot \\
z_{k+1,1}       & z_{k+1,2} & \cdot & z_{k+1,k+1} & 1
\end{bmatrix}.
\end{equation}
For $\ell \in [k+1]$ we also denote by $Z_\ell$ the matrix that consists of the top $\ell$ rows of $Z$.  Further, for a $(k+1)$-wise query function $\phi\big((z_1,b_1),\ldots,(z_{k+1},b_{k+1})  \big)$, we use $Z$ to refer to the matrix obtained from the inputs to the function.

Let $Q$ be the distribution defined by sampling a random example $\big((z_{1}, \dots, z_{k+1}),b\big)$, conditioning on the event that $b=1$ and outputting $(z_{1}, \dots, z_{k+1},1)$. Note that if the examples from which $Z$ is built are positively labeled i.i.d. examples then each row of $Z$ is sampled i.i.d. from $Q$ and hence $Z_\ell$ is distributed according to $Q^\ell$.
We denote by $1^{k+1}$ the all $+1$'s vector of length $k+1$.

\paragraph{Learning algorithm}
We start by explaining the main ideas behind the algorithm. On a high level, in order to be able to use $(k+1)$-wise SQs to learn the unknown subspace, we need to make sure that there exists an affine subspace that contains most of the probability mass of the positively-labeled points and
that is spanned by $k+1$ random positively-labeled points with noticeable probability. Here, the probability is with respect to the unknown distribution over labeled examples. Thus, for positively labeled tuples $(z_{1,1}, \dots, z_{1,k+1})$, $(z_{2,1}, \dots, z_{2,k+1})$, $\dots$, $(z_{k+1,1}, \dots, z_{k+1,k+1})$, we  consider the $(k+1) \times (k+2)$ matrix $Z$ defined in \Cref{eq:Z_mat_def}. If $W$ is the row-span of $Z$, then the desired (unknown) affine subspace is the set $V$ of all points $(z_1, \dots,z_{k+1})$ such that $(z_1, \dots,z_{k+1}, 1) \in W$.

	If the (unknown) distribution over labeled examples is such that with noticeable probability, $k+1$ random positively-labeled points form a full-rank linear system (i.e., the matrix $Z$ has full-rank with noticeable probability conditioned on $(b_1,\dots,b_{k+1}) = 1^{k+1}$), we can use $(k+1)$-wise SQs to find, one bit at a time, the $(k+1)$-dimensional row-span $W$ of $Z$, and we can then output the set $V$ of all points $(z_1, \dots,z_{k+1})$ such that $(z_1, \dots,z_{k+1}, 1) \in W$ as the desired affine subspace (below, we refer to this step as the Recovery Procedure).
		
	 We now turn to the (more challenging) case where the system is not full-rank with noticeable probability (i.e., the matrix $Z$ is rank-deficient with high probability conditioned on $(b_1,\dots,b_{k+1}) = 1^{k+1}$). Then, the system has rank at most $i$ with high probability, for some $i < k+1$. There is a large number of possible $i$-dimensional subspaces and therefore it is no longer clear that there exists a single $i$-dimensional subspace that contains most of the mass of the positively-labeled points. However, we demonstrate that for every $i$, if the rank of $Z$ is at most $i$ with sufficiently high probability, then there exists a \emph{fixed} subspace $W$ of dimension at most $i$ that contains a large fraction of the probability under the row-distribution of $Z$ (it turns out that if this subspace has rank equal to $i$, then it should be \emph{unique}). We can then use $(k+1)$-wise SQs to output the affine subspace $V$ consisting of all points $(z_1,\dots,z_{k+1})$ such that $(z_1,\dots,z_{k+1},1) \in W$ (via the Recovery Procedure).

	The general description of the algorithm is given in Algorithm~\ref{alg:k_wise_SQ}, and the Recovery Procedure (allowing the reconstruction of the affine subspace $V$) is separately described in \Cref{alg:recovery}. We denote the indicator function of event $E$ by $\ind(E)$. Note that the statistical query corresponding to the event $\ind(E)$ gives an estimate of the probability of $E$.
\begin{algorithm}[H]
\caption{$(k+1)$-wise SQ Algorithm}
\label{alg:k_wise_SQ}
{\bf Inputs.} $k \in \mathbb{N}$, error probability $\epsilon > 0$.\\
{\bf Output.} Function $f:\mathbb{F}_p^{k+1} \to \{\pm 1 \}$.
\begin{algorithmic}[1]
\State Set tolerance of each SQ to $\tau = (\epsilon/2^{c\cdot(k+2)})^{(k+1)^{k+3}}$, where $c>0$ is a large enough absolute constant.
\State Define the threshold $\tau_i = 2^{c \cdot (k+2-i)} \cdot k \cdot \tau^{1/(k+1)^{k+2-i}}$ for every $i \in [k+1]$.
\State Ask the SQ $\phi(z,b) \doteq \ind(b=1)$ and let $w$ be the response.
  \If{$w \le \epsilon -\tau$}
    \State\label{st:early_term} Output the all $-1$'s function.
  \EndIf
\State Let $\widetilde{\phi}\big((z_1,b_1),\ldots,(z_{k+1},b_{k+1}) \big) \doteq \ind((b_1,\dots,b_{k+1}) = 1^{k+1})$.
\State\label{st:v_resp} Ask the SQ $\widetilde{\phi}$ and let $v$ be the response.
\For{$i = k+1$ down to $1$}
\State Let $\phi_i\big((z_1,b_1),\ldots,(z_{k+1},b_{k+1})  \big) \doteq \ind((b_1,\dots,b_{k+1}) = 1^{k+1}\mbox{ and }\rk(Z) = i$).
	\State Ask the SQ $\phi_i$ and let $v_i$ be the response.
  \If{$v_i/v \geq \tau_i$}
	\State Run Recovery Algorithm on input $(i,v_i)$ and let $\widehat{V}$ be the subspace of $\mathbb{F}_p^{k+1}$ it outputs.
	      		\State Define function $f:\mathbb{F}_p^{k+1} \to \{-1,1\}$ by:
	\State $f(z_1,\dots,z_{k+1}) = +1$ if $(z_1,\dots,z_{k+1}) \in \widehat{V}$.
	\State $f(z_1,\dots,z_{k+1}) = -1$ otherwise.
	\State Return $f$.
  \EndIf
      \EndFor

\end{algorithmic}
\end{algorithm}
\begin{algorithm}[H]
	\caption{Recovery Procedure}
	\label{alg:recovery}
	{\bf Input.} Integer $i \in [k+1]$.\\
	{\bf Output.} Subspace $\widehat{V}$ of $\mathbb{F}_p^{k+1}$ of dimension $i$.
	\begin{algorithmic}[1]
         \State Let $m_i = (k+2) \cdot i \cdot \lceil \log p \rceil$
		 \For{each bit $j \leq m_i$}
		 \State Define event $E_j(Z) = \ind(\text{bit } j \text{ of row span of } Z \text{ is } 1)$.
		 \State Let $\phi_{i,j} \big((z_1,b_1),\ldots,(z_{k+1},b_{k+1})  \big) \doteq \ind( E_j(Z) \mbox{ and } (b_1,\dots,b_{k+1}) = 1^{k+1}\mbox{ and }\rk(Z) = i$).
		 \State Ask the SQ $\phi_{i,j}$ and let $u_{i,j}$ be the response.
		
		    \If{$u_{i,j}/v_i \geq (9/10)$}
		    \State Set bit $j$ in binary representation of $\widehat{W}$ to $1$.
		    \Else
		    \State Set bit $j$ in binary representation of $\widehat{W}$ to $0$.
		    \EndIf
		  \EndFor
		  \State Let $\widehat{V}$ be the set all points $(z_1,\dots,z_{k+1})$ such that $(z_1,\dots,z_{k+1},1) \in \widehat{W}$.
	\end{algorithmic}
\end{algorithm}



\paragraph{Analysis}
We now turn to the analysis of Algorithm~\ref{alg:k_wise_SQ} and the proof of Lemma~\ref{le:k_ub}. We will need the following lemma, which shows that if the rank of $Z$ is at most $i$ with high probability, then there is a \emph{fixed} subspace of dimension at most $i$ containing most of the probability mass under the row-distribution of $Z$.
\begin{lem}\label{le:ex_sub}
Let $i \in [k+1]$. If $\Pr_{Q^{k+1}}[\rk(Z) \le i] \geq 1-\xi$, then there exists a subspace $W$ of $\mathbb{F}_p^{k+2}$ of dimension at most $i$ such that $\Pr_{z \sim Q}[z \notin W] \le \xi^{1/k}$.
\end{lem}

\begin{remark}
We point out that the exponential dependence on $1/k$ in the probability upper bound in Lemma~\ref{le:ex_sub} is tight. To see this, let $p = 2$, and $\{e_1, \dots , e_k\}$ be the standard basis in $\mathbb{F}_2^k$. Consider the base distribution $P$ on $\mathbb{F}_2^k$ that puts probability mass $1-\alpha$ on $e_1$, and probability mass $\alpha/(k-1)$ on each of $e_2$, $e_3$, $\dots$, $e_k$. Then, a Chernoff bound implies that if we draw $k$ i.i.d. samples from $P$, then the dimension of their span is at most $2 \cdot \alpha \cdot k$ with probability at least $1 - \exp(-k)$. On the other hand, for any subspace $W$ of $\mathbb{F}_2^k$ of dimension $2 \cdot \alpha \cdot k$, the probability that a random sample from $P$ lies inside $W$ is only $1- \Theta(\alpha)$.
\end{remark}

To prove Lemma~\ref{le:ex_sub}, we will use the following proposition.
\begin{proposition}\label{prop:ind}
Let $\ell \in [k+1]$, $i \in [\ell-1]$ and $\eta >0$. If $\Pr_{Q^{\ell}}[\rk(Z_{\ell}) \le i] \geq 1-\eta$, then for every $\nu \in (0,1]$, either there exists a subspace $W$ of $\mathbb{F}_p^{k+2}$ of dimension $i$ such that $\Pr_{z \sim Q}[z \notin W] \le \nu$ or $\Pr_{Q^i}[\rk(Z_{i}) \le i-1] \geq 1-\eta/\nu$.
\end{proposition}

\begin{proof}
Let $p \doteq \Pr_{Q^i}[\rk(Z_{i}) \le i-1]$. For every (fixed) matrix $A_i \in \mathbb{F}_p^{i \times (k+2)}$, define
$$\mu(A_i) \doteq \Pr_{Q^\ell}[\rk(Z_{\ell}) \le i ~ | ~ Z_{i} = A_i].$$
Then,
\begin{align*}
\Pr_{Q^{\ell}}[\rk(Z_{\ell}) \le i] &= p+(1-p)\cdot \Pr_{Q^{\ell}}[\rk(Z_{\ell}) \le i ~ | ~ \rk(Z_{i}) = i]\\
&= p+(1-p)\cdot \Ex_{ Q^i}\bigg[\mu(Z_i) \bigg| ~ \rk(Z_{i}) = i \bigg].
\end{align*}
Since $\Pr_{Q^{\ell}}[\rk(Z_{\ell}) \le i] \geq 1-\eta$, we have that
$$  \Ex_{ Q^i}\bigg[\mu(Z_i) \bigg| ~ \rk(Z_{i}) = i \bigg] \geq 1 - \eta/(1-p). $$
Hence, there exists a setting $A_i \in \mathbb{F}_p^{i \times (k+2)}$ of $Z_{i}$  such that $\rk(A_{i}) = i$ and
$$\Pr[\rk(Z_{\ell}) \le i ~ | ~ Z_{i} = A_{i}] \geq 1 - \eta/(1-p).$$
We let $W$ be the $\mathbb{F}_p$-span of the rows of $A_{i}$. Note that the dimension of $W$ is equal to $i$ and that $\Pr_{z \sim Q}[z \notin W] \le \eta/(1-p)$. Thus, we conclude that for every $\nu \in (0,1]$, either $p \geq 1-\eta/\nu$ or $\Pr_{z \sim Q}[z \notin W] \le \nu$, as desired.
\end{proof}

We now complete the proof of Lemma~\ref{le:ex_sub}.
\begin{proof}[Proof of Lemma~\ref{le:ex_sub}]
Starting with $\ell = k+1$ and $\eta = \xi$, we inductively apply Proposition~\ref{prop:ind} with $\nu = \xi^{1/k}$ until we either get the desired subspace $W$ or we get to the case where $i=1$. In this case, we have that $\Pr_{Q^{\ell}}[\rk(Z_{\ell}) \le 1] \geq 1-\xi^{1/k}$ for $\ell \geq 2$. Since the last column of $Z_{\ell}$ is the all $1$'s vector, we conclude that there exists $z^* \in \mathbb{F}_p^{k+1}$ such that $\Pr_{z \sim Q}[z \neq (z^*,1)] \le \xi^{1/k}$. We can then set our subspace $W$ to be the $\mathbb{F}_p$-span of the vector $(z^*,1)$.
\end{proof}

For the proof of Lemma~\ref{le:k_ub} we will also need the following lemma, which states sufficient conditions under which the Recovery Procedure (\Cref{alg:recovery}) succeeds.
\begin{lem}\label{le:recovery}
	Let $i \in [k+1]$. Assume that in \Cref{alg:k_wise_SQ}, $v > \epsilon^{k+1}/2$ and $v_i/v \geq \tau_i$. If there exists a subspace $W$ of $\mathbb{F}_p^{k+2}$ of dimension equal to $i$ such that
	\begin{equation}\label{eq:lemma_W_assumption}
	\Pr_{z \sim Q}[z \notin W] < \frac{\tau_i} {4 \cdot (k+1)},
	\end{equation}
	then the affine subspace $\widehat{V}$ output by \Cref{alg:recovery} (i.e., the Recovery Procedure) consists of all points $(z_1,\dots,z_{k+1})$ such that $(z_1,\dots,z_{k+1},1) \in W$.
\end{lem}

We note that \Cref{le:recovery} would still hold under quantitatively weaker assumptions on $v$, $v_i/v$ and $\Pr_{z \sim Q}[z \notin W]$ in \Cref{eq:lemma_W_assumption}. In order to keep the expressions simple, we however choose to state the above version which will be sufficient to prove \Cref{le:k_ub}. The proof of \Cref{le:recovery} appears in \Cref{subsec:pf_rec_lem}. We are now ready to complete the proof of Lemma~\ref{le:k_ub}.

\begin{proof}[Proof of Lemma~\ref{le:k_ub}]
If Algorithm~\ref{alg:k_wise_SQ} terminates at Step~\ref{st:early_term}, then the error of the output hypothesis is at most $\epsilon$, as desired. Henceforth, we assume that Algorithm~\ref{alg:k_wise_SQ} does not terminate at Step~\ref{st:early_term}. Then, we have that $\Pr[b = 1] > \epsilon$, and hence $\Pr[(b_1,\dots,b_{k+1}) = 1^{k+1}] > \epsilon^{k+1}$. Thus, the value $v$ obtained in Step~\ref{st:v_resp} of Algorithm~\ref{alg:k_wise_SQ} satisfies $v > \epsilon^{k+1} - \tau \geq \epsilon^{k+1}/2$, where the last inequality follows from the setting of $\tau$. Let $i^*$ be the first (i.e., largest) value of $i \in  [k+1]$ for which $v_i/v \geq \tau_i$. To prove that such an $i^*$ exists, we proceed by contradiction, and assume that for all $i \in [k+1]$, it is the case that $v_i/v < \tau_i$. Note that $Z$ has an all $1$'s column, so it has rank at least $1$. Moreover, it has rank at most $k+1$. Therefore, we have that
\begin{align*}
1 &= \Pr[1 \le \rk(Z) \le k+1 ~ |~ (b_1,\dots,b_{k+1})=1^{k+1}]\\
&= \displaystyle\sum\limits_{i=1}^{k+1} \Pr[\rk(Z) = i ~ | ~ (b_1,\dots,b_{k+1})=1^{k+1}]\\
&\le \displaystyle\sum\limits_{i=1}^{k+1} \frac{v_i + \tau}{v - \tau}\\
&\le 2 \cdot \displaystyle\sum\limits_{i=1}^{k+1} \frac{v_i + \tau}{v}\\
&\le 2 \cdot \displaystyle\sum\limits_{i=1}^{k+1} (\frac{v_i}{v}  + \frac{2\tau}{\epsilon^{k+1}})\\
&< 2 \cdot \displaystyle\sum\limits_{i=1}^{k+1} \tau_i + 4 \cdot (k+1) \cdot \frac{\tau}{\epsilon^{k+1}}.
\end{align*}
Using the fact that $\tau_i$ is monotonically non-increasing in $i$ and the settings of $\tau_1$ and $\tau$, the last inequality gives
\begin{align*}
1 &\le 2 \cdot (k+1) \cdot \tau_1 + 4 \cdot (k+1) \cdot \frac{\tau}{\epsilon^{k+1}} < 1,
\end{align*}
a contradiction.

We now fix $i^*$ as above. We have that
\begin{align*}
\Pr[\rk(Z) \le i^* ~ | ~ (b_1,\dots,b_{k+1})=1^{k+1}] &= 1 - \displaystyle\sum\limits_{i = i^*+1}^{k+1} \Pr[\rk(Z) = i ~ | ~ (b_1,\dots,b_{k+1})=1^{k+1}]\\
&\geq 1 - \displaystyle\sum\limits_{i = i^*+1}^{k+1} \frac{v_i + \tau}{v-\tau}\\
&\geq 1 - 2 \cdot \displaystyle\sum\limits_{i = i^*+1}^{k+1} (\frac{v_i}{v}  + \frac{2\tau}{\epsilon^{k+1}})\\
&> 1 - 2 \cdot \displaystyle\sum\limits_{i = i^*+1}^{k+1} (\tau_i + 2 \cdot \frac{\tau}{\epsilon^{k+1}})\\
& \geq 1 - 4 \cdot \displaystyle\sum\limits_{i = i^*+1}^{k+1} \tau_i\\
&\geq 1 - 4 \cdot k\cdot \tau_{i^*+1}.
\end{align*}
By Lemma~\ref{le:ex_sub}, there exists a subspace $W$ of $\mathbb{F}_p^{k+2}$ of dimension at most $i^*$ such that
\begin{equation}\label{eq:notin_W}
\Pr_{z \sim Q}[z \notin W] \le (4 \cdot k)^{1/k} \cdot \tau_{i^*+1}^{1/k}.
\end{equation}

\begin{proposition}\label{prop:failure}
	For every $i \in [k]$, we have that $(k+1) \cdot (4 \cdot k)^{1/k} \cdot \tau_{i+1}^{1/k} \le \tau_{i }/4$.
\end{proposition}
We note that \Cref{prop:failure} follows immediately from the definitions of $\tau_{i}$ and $\tau$ (and by letting $c$ by a sufficiently large positive absolute constant). Moreover, \Cref{prop:failure} (applied with $i = i^*$) along with \Cref{eq:notin_W} imply that $\Pr_{z \sim Q}[z \notin W]$ is at most $\tau_{i*}/(4(k+1))$.

By a union bound, we get that with probability at least
\begin{equation}\label{eq:alg_succ_prob}
1- (k+1) \cdot \Pr_{z \sim Q}[z \notin W] \geq 1 - \frac{\tau_{i^*}}{4},
\end{equation}
all the rows of $Z$ belong to $W$.

Since $v_{i*}/v \geq \tau_{i*}$, we also have that:
\begin{align}\label{eq:cond_lb_i_star}
\Pr[\rk(Z) = i^* ~ | ~ (b_1,\dots,b_{k+1})=1^{k+1}] &\geq \frac{v_{i*}-\tau}{v + \tau}\nonumber\\
&\geq \frac{1}{2} \cdot \frac{(v_{i*}-\tau)}{v}\nonumber\\
&\geq \frac{1}{2} \cdot (\tau_{i^*} - \frac{2 \cdot \tau}{\epsilon^{k+1}})\nonumber\\
& \geq \frac{\tau_{i^*}}{3}
\end{align}
Combining \Cref{eq:alg_succ_prob} and \Cref{eq:cond_lb_i_star}, we get that the rank of $W$ is \emph{equal to} $i^*$.

Let $V$ be the affine subspace consisting of all points $(z_1,\dots,z_{k+1})$ such that $(z_1,\dots,z_{k+1},1) \in W$. By \Cref{le:recovery}, we get that \Cref{alg:recovery} (and hence \Cref{alg:k_wise_SQ}) correctly recovers the affine subspace $V$.

We note that the function $f$ output by Algorithm~\ref{alg:k_wise_SQ} is the $\pm 1$ indicator of a subspace of the true hyperplane $\Hyp_a$. To see this, note that $f$ is the $\pm 1$ indicator function of the subspace $V$, and by Equations~(\ref{eq:notin_W}) and (\ref{eq:cond_lb_i_star}), we have that with probability at least $\tau_{i*}/12$ over $Z \sim Q^{k+1}$, all the columns of $Z$ belong to $W$ and $\rk(Z) = i^*$. Since the dimension of $W$ is equal to $i^*$ and since we are conditioning on $(b_1,\dots,b_{k+1})=1^{k+1}$, this implies that the correct label of all the points in $V$ is $+1$. Hence, $f$ only possibly errs on positively-labeled points (by wrongly giving them the label $-1$). Moreover, Algorithm~\ref{alg:k_wise_SQ} ensures that the output function $f$ gives the label $+1$ to every $(z_1,\dots,z_{k+1}) \in \mathbb{F}_p^{k+1}$ for which $(z_1,\dots,z_{k+1},1) \in W$. Therefore, the function $f$ that is output by Algorithm~\ref{alg:k_wise_SQ} (when it does not terminate at Step~\ref{st:early_term}) has error at most the right hand side of (\ref{eq:notin_W}). So to upper-bound the error probability, it suffices for us to verify that the right-hand side of (\ref{eq:notin_W}) is at most $\epsilon$. This is obtained by applying the next proposition with $i = i^*+1$.
\begin{proposition}\label{prop:tau_i_pow_1_ov_k}
	For every $i \in [k+1]$, we have that $(4 \cdot k)^{1/k} \cdot \tau_{i}^{1/k} \le \epsilon^{k} $.
\end{proposition}

The proof of \Cref{prop:tau_i_pow_1_ov_k} follows immediately from the definitions of $\tau_{i}$ and $\tau$ and by letting $c$ be a sufficiently large positive absolute constant.

The number of queries performed by the $(k+1)$-wise algorithm is at most $O(k^2 \cdot \log{p})$, and their tolerance is $\tau \geq (\epsilon/2^{c\cdot(k+2)})^{(k+1)^{k+3}}$, where $c$ is a positive absolute constant. Finally, we remark that the dependence of the SQ complexity of the above algorithm on the error parameter $\eps$ is $\eps^{-k^{O(k)}}$. It can be improved to a linear dependence on $1/\eps$ by learning with error $1/3$ and then using boosting in the standard way (boosting in the SQ model works essentially as in the regular PAC model \cite{aslam1993general}).
\end{proof}

\subsection{Lower bound}
Our proof of lower bound is a generalization of the lower bound in \cite{Feldman:16sqd} (for $\ell=2$ and $k=1$). It relies on a notion of {\em combined randomized statistical dimension} (``combined" refers to the fact that it examines a single parameter that lower bounds both the number of queries and the inverse of the tolerance).
In order to apply this approach we need to extend it to $k$-wise queries. This extension follows immediately from a simple observation. If we define the domain to be $X' \doteq X^k$ and the input distribution to be $D' \doteq D^k$ then asking a $k$-wise query $\phi:X^k \to [-1,1]$ to $\STAT^{(k)}_D(\tau)$ is equivalent to asking a unary query $\phi: X' \to [-1,1]$ to $\STAT^{(k)}_{D'}(\tau)$. Using this observation we define the $k$-wise versions of the notions from \cite{Feldman:16sqd} and give their properties that are needed for the proof of Lemma~\ref{le:k_plus_1_lb}.

\subsubsection{Preliminaries}
Combined randomized statistical dimension is based on the following notion of average discrimination.
\begin{defn}[$k$-wise average $\kappa_1$-discrimination]\label{def:kappa_1_disc}
Let $k$ be any positive integer. Let $\mu$ be a probability measure over distributions over $X$ and $D_0$ be a reference distribution over $X$. Then,
$$ \bar{\kappa}_1^{(k)}(\mu,D_0) \doteq \sup_{\phi: X^k \to [-1,+1]} \bigg\{ \Ex_{D \sim \mu}[|D^k[\phi]-D_0^k[\phi]|] \bigg\}. $$
\end{defn}

We denote the problem of PAC learning a concept class $\calC$ of Boolean functions up to error $\epsilon$ by $\mathcal{L}_{PAC}(\calC,\epsilon)$. Let $Z$ be the domain of the Boolean functions in $\calC$. For any distribution $D_0$ over labeled examples (i.e., over $Z \times \{\pm 1\}$), we define the Bayes error rate of $D_0$ to be
\begin{equation*}
\err(D_0) = \displaystyle\sum\limits_{z \in Z} \min\{D_0(z,1) , D_0(z,-1)\} = \min_{h: Z \to \{\pm1 \}} \Pr_{(z,b) \sim D_0}[h(z) \neq b].
\end{equation*}

\begin{defn}[$k$-wise combined randomized statistical dimension]\label{def:csdr}
Let $k$ be any positive integer. Let $\calD$ be a set of distributions and $D_0$ a reference distribution over $X$. The $k$-wise combined randomized statistical dimension of the decision problem $\calB(\calD,D_0)$ is then defined as
$$ \CSDR_{\bar{\kappa}_1}^{(k)}(\calB(\calD,D_0)) \doteq \sup_{\mu \in S^{\calD}} (\bar{\kappa}_1^{(k)}(\mu,D_0))^{-1}, $$
where $S^\D$ denotes the set of all probability distributions over $\D$.

Further, for any concept class $\calC$ of Boolean functions over a domain $Z$, and for any $\epsilon > 0$, the $k$-wise combined randomized statistical dimension of $\mathcal{L}_{PAC}(\calC,\epsilon)$ is defined as
\begin{equation*}
\CSDR_{\bar{\kappa}_1}^{(k)}(\mathcal{L}_{PAC}(\calC,\epsilon)) \doteq \sup_{D_0 \in S^{Z \times \{\pm 1\}}: \err(D_0) > \epsilon} \CSDR_{\bar{\kappa}_1}^{(k)}(\calB(\calD_{\calC},D_0)),
\end{equation*}
where $\calD_{\calC} \doteq \{ P^f: P \in S^{Z}, f \in \calC\}$ with $P^f$ denoting the distribution on labeled examples $(x,f(x))$ with $x \sim P$.
\end{defn}

The next theorem lower bounds the randomized $k$-wise SQ complexity of PAC learning a concept class in terms of its $k$-wise combined randomized statistical dimension.

\begin{theorem}[\cite{Feldman:16sqd}]\label{thm:sq_RSD}
Let $\calC$ be a concept class of Boolean functions over a domain $Z$, $k$ be a positive integer and $\epsilon, \delta > 0$. Let $d \doteq \CSDR_{\bar{\kappa}_1}^{(k)}(\mathcal{L}_{PAC}(\calC,\epsilon))$. Then, the randomized $k$-wise SQ complexity of solving $\mathcal{L}_{PAC}(\calC,\epsilon - 1/\sqrt{d})$ with access to $\STAT^{(k)}(1/\sqrt{d})$ and success probability $1-\delta$ is at least $(1-\delta) \cdot \sqrt{d} - 1$.
\end{theorem}

To lower bound the statistical dimension we will use the following ``average correlation'' parameter introduced in \cite{FeldmanGRVX:12}.
\begin{defn}[$k$-wise average correlation]\label{def:rho}
Let $k$ be any positive integer. Let $\calD$ be a set of distributions and $D_0$ a reference distribution over $X$. Assume that the support of every distribution $D \in \calD$ is a subset of the support of $D_0$. Then, for every $x \in X^k$, define $\hat{D}(x) \doteq \frac{D^k(x)}{D_0^k(x)} - 1$. Then, the $k$-wise average correlation is defined as
$$ \rho^{(k)}(\calD,D_0) \doteq \frac{1}{|\calD|^2} \cdot \displaystyle\sum\limits_{D, D' \in \calD} | D_0^k[\hat{D} \cdot \hat{D}']|. $$
\end{defn}

Lemma~\ref{lem:ub_rho} relates the average correlation to the average discrimination (from Definition~\ref{def:kappa_1_disc}).
\begin{lem}[\cite{Feldman:16sqd}]\label{lem:ub_rho}
Let $k$ be any positive integer. Let $\calD$ be a set of distributions and $D_0$ a reference distribution over $X$. Let $\mu$ be the uniform distribution over $\calD$. Then,
$$ \bar{\kappa}_1^{(k)}(\mu,D_0) \le 4 \cdot \sqrt{\rho^{(k)}(\calD,D_0)}. $$
\end{lem}

\subsubsection{Proof of Lemma~\ref{le:k_plus_1_lb}}\label{subsec:sep_lb}

Denote $X \doteq \mathbb{F}_p^{\ell} \times \{\pm 1\}$. Let $\calD$ be the set of all distributions over $X^k$ that are obtained by sampling from any given distribution over $(\mathbb{F}_p^{\ell})^k$ and labeling the $k$ samples according to any given hyperplane indicator function $f_a$. Let $D_0$ be the uniform distribution over $X^k$. We now show that $\CSDR_{\bar{\kappa}_1}(\calB(\calD,D_0)) = \Omega\big(p^{(\ell-k)/2}\big)$. By definition,
$$ \CSDR_{\bar{\kappa}_1}(\calB(\calD,D_0)) \doteq \sup_{\mu \in S^{\calD}} (\bar{\kappa}_1(\mu,D_0))^{-1}. $$
We now choose the distribution $\mu$. For $a \in \mathbb{F}_p^{\ell}$, we define $P_a$ to be the distribution over $\mathbb{F}_p^{\ell}$ that has density $\alpha = 1/(2 (p^{\ell}-p^{\ell-1}))$ on each of the $p^{\ell}-p^{\ell-1}$ points outside $\Hyp_a$, and density $\beta = 1/p^{\ell-1}-\alpha p +\alpha = 1/(2p^{\ell-1})$ on each of the $p^{\ell-1}$ points inside $\Hyp_a$. We then define $D_a$ to be the distribution obtained by sampling $k$ i.i.d.~random examples of $\Hyp_a$, the marginal of each over $\mathbb{F}_p^{\ell}$ being $P_a$. Let $\calD' \doteq \{D_a ~ | ~ a \in \mathbb{F}_p^{\ell}\}$, and let $\mu$ be the uniform distribution over $\calD'$. By Lemma~\ref{lem:ub_rho}, we have that $\bar{\kappa}_1(\mu,D_0) \le 4 \cdot \sqrt{\rho(\calD,D_0)}$, so it is enough to upper bound $\rho(\calD,D_0)$.

We first note that for $a, a' \in \mathbb{F}_p^{\ell}$, we have
\begin{align*}
D_0[\hat{D}_a \cdot \hat{D}_{a'}] &= \Ex_{(z,b) \sim D_0} [ \hat{D}_a(z,b) \cdot \hat{D}_{a'}(z,b)]\\
&= \Ex_{(z,b) \sim D_0} \bigg[ \bigg(\frac{D_a(z,b)}{D_0(z,b)}-1 \bigg) \cdot \bigg(\frac{D_{a'}(z,b)}{D_0(z,b)}-1\bigg)\bigg]\\
&= \Ex_{(z,b) \sim D_0} \bigg[ \frac{D_a(z,b) \cdot D_{a'}(z,b)}{D_0^2(z,b)} - \frac{D_a(z,b)}{D_0(z,b)} - \frac{D_{a'}(z,b)}{D_0(z,b)} +1\bigg]\\
&= \Ex_{(z,b) \sim D_0} \bigg[ \frac{D_a(z,b) \cdot D_{a'}(z,b)}{D_0^2(z,b)}\bigg] - 2 \cdot \Ex_{(z,b) \sim D_0} \bigg[\frac{D_a(z,b)}{D_0(z,b)}\bigg] +1\\
&= 2^{2k} \cdot p^{2 k \ell} \cdot \Ex_{(z,b) \sim D_0}[D_a(z,b) \cdot D_{a'}(z,b)] - 2^{k+1} \cdot p^{k \ell} \cdot \Ex_{(z,b) \sim D_0}[D_a(z,b)] + 1
\end{align*}

We now compute each of the two expectations that appear in the last equation above.

\begin{proposition}\label{prop:single_ex}
For every $a \in \mathbb{F}_p^{\ell}$,
$$ \Ex_{(z,b) \sim D_0}[D_a(z,b)] = \frac{1}{2^k} \cdot \bigg(\frac{1}{p} \cdot \beta + \bigg(1-\frac{1}{p}\bigg) \cdot \alpha\bigg)^k = \frac{1}{2^k \cdot p^{k\cdot \ell}}.$$
\end{proposition}

The proof of Proposition~\ref{prop:single_ex} appears in the appendix.

\begin{proposition}\label{prop:pair_ex}
For every $a, a' \in \mathbb{F}_p^{\ell}$,
\[
\Ex_{(z,b) \sim D_0}[D_a(z,b) \cdot D_{a'}(z,b)] = \begin{dcases*}
	& $\frac{1}{2^k} \cdot (\frac{1}{p} \cdot \beta^2 + (1-\frac{1}{p}) \cdot \alpha^2)^k$ if $\Hyp_a = \Hyp_{a'}$,\\
	& $\frac{1}{2^k}\cdot (\alpha^2 \cdot (1-\frac{2}{p}))^k$ if $\Hyp_a \cap \Hyp_{a'} = \emptyset$,\\
	& $\frac{1}{2^k} \cdot ( \frac{\beta^2}{p^2} +\alpha^2 \cdot (1-\frac{2}{p}+\frac{1}{p^2}))^k$ otherwise.
        \end{dcases*}
\]
\end{proposition}

The proof of Proposition~\ref{prop:pair_ex} appears in the appendix. Using Proposition~\ref{prop:single_ex} and Proposition~\ref{prop:pair_ex}, we now compute $D_0[\hat{D}_a \cdot \hat{D}_{a'}]$.

\begin{proposition}\label{prop:D_0}
For every $a, a' \in \mathbb{F}_p^{\ell}$,
\[
D_0[\hat{D}_a \cdot \hat{D}_{a'}] = \begin{dcases*}
	& $(p+1-\frac{1}{p-1})^k - 1$ if $\Hyp_a = \Hyp_{a'}$,\\
	& $\frac{1}{2^k} \cdot \frac{(1-\frac{2}{p})^k}{(1-\frac{1}{p})^{2k}}-1$ if $\Hyp_a \cap \Hyp_{a'} = \emptyset$,\\
	& $0$ otherwise.
        \end{dcases*}
\]
\end{proposition}

The proof of Proposition~\ref{prop:D_0} appears in the appendix. When computing $\rho(\calD,D_0)$, we will also use the following simple proposition.
\begin{proposition}\label{prop:pairs_hyp}
\begin{enumerate}
\item The number of pairs $(a,a') \in (\mathbb{F}_p^{\ell})^2$ such that $\Hyp_a = \Hyp_{a'}$ is equal to $p^{\ell}$.
\item The number of pairs $(a,a') \in (\mathbb{F}_p^{\ell})^2$ such that $\Hyp_a$ and $\Hyp_{a'}$ are distinct and parallel is equal to $p^{\ell}\cdot(p-1)$.
\item The number of pairs $(a,a') \in (\mathbb{F}_p^{\ell})^2$ such that $\Hyp_a$ and $\Hyp_{a'}$ are distinct and intersecting is equal to $p^{2\cdot \ell}-p^{\ell+1}$.
\end{enumerate}
\end{proposition}

Using Proposition~\ref{prop:D_0} and Proposition~\ref{prop:pairs_hyp}, we are now ready to compute $\rho(\calD,D_0)$ as follows
\begin{align*}
\rho(\calD,D_0) &\le \frac{1}{p^{2\cdot \ell}} \cdot \bigg[ p^{\ell} \cdot (p+1-\frac{1}{p-1})^k +p^{\ell} \cdot (p-1) + p^{2\cdot \ell} \cdot 0 \bigg]\\
&\le O\bigg(\frac{1}{p^{\ell-k}}\bigg) + \frac{1}{p^{\ell-1}}\\
&= O\bigg(\frac{1}{p^{\ell-k}}\bigg),
\end{align*}
where we used above the assumption that $k = O(p)$. We deduce that $\bar{\kappa}_1(\mu,D_0)  = O\bigg(1/p^{(\ell-k)/2}\bigg)$, and hence $\CSDR_{\bar{\kappa}_1}(\calB(\calD,D_0)) = \Omega\bigg(p^{(\ell-k)/2}\bigg)$. This lower bound on $\CSDR_{\bar{\kappa}_1}(\calB(\calD,D_0))$, along with Definition~\ref{def:csdr}, Theorem~\ref{thm:sq_RSD} and the fact that $D_0$ has Bayes error rate equal to $1/2$, imply Lemma~\ref{le:k_plus_1_lb}.

\section{Reduction for flat distributions}\label{sec:pf_flat}
\newcommand{\dci}{{\kappa_1}}
To prove Theorem~\ref{thm:flat} we use the characterization of the SQ complexity of the problem of estimating $D^k[\phi]$ for $D\in \D$ using a notion of statistical dimension from \cite{Feldman:16sqd}. Specifically, we use the characterization of the complexity of solving this problem using unary SQs and also the generalization of this characterization that characterizes the complexity of solving a problem using $k$-wise SQs. The latter is equal to 1 (since a single $k$-wise SQ suffices to estimate $D^k[\phi]$). Hence the $k$-wise statistical dimension is also equal to 1. We then upper bound the unary statistical dimension by the $k$-wise statistical dimension. The characterization then implies that an upper bound on the unary statistical dimension gives an upper bound on the SQ complexity of estimating $D^k[\phi]$.

We also give a slightly different way to define flatness that makes it easier to extend our results to other notions of divergence.
\begin{defn}
Let $\calD$ be a set of distributions over $X$. Define
$$R_\infty(\D) \doteq \inf_{\bar D \in S^X} \sup_{D\in \D} \Div_\infty(D\|\bar D), $$
where $S^X$ denotes the set of all probability distributions over $X$ and $$\Div_\infty(D\|\bar D) \doteq \sup_{y\in X} \ln \frac{\Pr_{x\sim D}[x=y]}{\Pr_{x\sim \bar D}[x=y]}$$ denotes the max-divergence. We say that $\calD$ is $\gamma$-flat if $R_\infty(\D) \leq \ln \gamma$.
\end{defn}

For simplicity, we will start by relating the $k$-wise SQ complexity to unary SQ complexity for decision problems. The statistical dimension for this type of problems is substantially simpler than for the general problems but is sufficient to demonstrate the reduction. We then build on the results for decision problems to obtain the proof of Theorem~\ref{thm:flat}.

\subsection{Decision problems}
The $k$-wise generalization of the statistical dimension for decision problems from \cite{Feldman:16sqd} is defined as follows.
\begin{defn}
Let $k$ be any positive integer. Consider a set of distributions $\calD$ and a reference distribution $D_0$ over $X$. Let $\mu$ be a probability measure over $\calD$ and let $\tau > 0$. The $k$-wise maximum covered $\mu$-fraction is defined as
$$ \kappa_1\text{-}\fract^{(k)}(\mu,D_0,\tau) \doteq \sup_{\phi: X^k \to [-1,+1]} \bigg\{ \Pr_{D \sim \mu}[|D^k[\phi]-D_0^k[\phi]| > \tau] \bigg\}. $$
\end{defn}
\begin{defn}[$k$-wise randomized statistical dimension of decision problems]\label{def:rdm_sd}
Let $k$ be any positive integer. For any set of distributions $\cald$, a reference distribution $D_0$ over $X$ and $\tau > 0$, we define
$$ \RSD_{\kappa_1}^{(k)}(\calB(\calD,D_0), \tau) \doteq \sup_{\mu \in S^{\calD}} ( \kappa_1\text{-}\fract^{(k)}(\mu,D_0,\tau))^{-1}, $$
where $S^\D$ denotes the set of all probability distributions over $\D$.
\end{defn}

As shown in \cite{Feldman:16sqd}, $\RSD$ tightly characterizes the randomized statistical query complexity of solving the problem using $k$-wise queries. As observed before, the $k$-wise versions below are implied by the unary version in \cite{Feldman:16sqd} simply by defining the domain to be $X' \doteq X^k$ and the set of input distributions to be $\D' \doteq \{D^k \ |\  D \in \D\}$.

\begin{theorem}[\cite{Feldman:16sqd}]\label{thm:random-algorithm2queries}
Let $\calB(\D,D_0)$ be a decision problem, $\tau > 0, \delta \in (0,1/2)$, $k \in \mathbb{N}$ and $d=\RSD^{(k)}_\dci(\calB(\D,D_0),\tau)$. Then there exists a randomized algorithm that solves $\calB(\D,D_0)$ with success probability $\geq 1-\delta$ using $d \cdot \ln(1/\delta)$ queries to $\STAT^{(k)}_D(\tau/2)$. Conversely, any algorithm that solves $\calB(\D,D_0)$ with success probability $\geq 1-\delta$ requires at least $d \cdot (1-2\delta)$ queries to $\STAT^{(k)}_D(\tau)$.
\end{theorem}

We will also need the following dual formulation of the statistical dimension given in Theorem~\ref{def:rdm_sd}.
\begin{lem}[\cite{Feldman:16sqd}]\label{fa:rcvr}
Let $k$ be any positive integer. For any set of distributions $\cald$, a reference distribution $D_0$ over $X$ and $\tau > 0$,  the statistical dimension $\RSD_{\kappa_1}^{(k)}(\calB(\calD,D_0), \tau)$ is equal to the smallest $d$ for which there exists a distribution $\calP$ over functions from $X^k$ to $[-1,+1]$ such that for every $D \in \calD$,
$$ \Pr_{\phi \sim \calP}[|D^k[\phi]-D_0^k[\phi]| > \tau] \geq \frac{1}{d}.$$
\end{lem}

We can now state the relationship between $\RSD_{\kappa_1}^{(k)}$ and $\RSD_{\kappa_1}^{(1)}$ for any $\gamma$-flat $\D$.
\begin{lem}\label{lem:k-wise-flat-decision}
Let $\gamma \geq 1$, $\tau > 0$ and $k \in \mathbb{N}$. Let $X$ be a domain, $\calD$ be a $\gamma$-flat class of distributions over $X$ and $D_0$ be any distribution over $X$. Then
$$\RSD_{\kappa_1}^{(1)}(\calB(\calD,D_0),\tau/(2k))  \leq \frac{4k \cdot \gamma^{k-1}}{\tau} \cdot \RSD_{\kappa_1}^{(k)}(\calB(\calD,D_0),\tau).$$
\end{lem}
\begin{proof}
Let $d \doteq \RSD_{\kappa_1}^{(k)}(\calB(\calD,D_0),\tau)$. Fact~\ref{fa:rcvr} implies the existence of a distribution $\calP$ over $k$-wise functions such that for every $D \in \calD$,
$$\Pr_{\phi \sim \calP}[|D^k[\phi]-D_0^k[\phi]| > \tau] \geq \frac{1}{d}.$$
We now fix $D$ and let $\phi$ be such that $|D^k[\phi]-D_0^k[\phi]| > \tau$. 

By the standard hybrid argument,  
\begin{equation}\label{eq:good_j}
\E_{j \sim [k]} \left[\left|D^{j} D_0^{k-j}[\phi]-D^{j-1}D_0^{k-j+1}[\phi]\right|\right] > \frac{\tau}{k} ,
\end{equation}
where $j \sim [k]$ denotes a random and uniform choice of $j$ from $[k]$.
This implies that
\begin{equation*}\label{eq:pull_out}
\E_{j \sim [k]} \Ex_{x_{< j} \sim D^{j-1}} \Ex_{x_{> j} \sim D_0^{k-j}} \bigg[\bigg|D[\phi(x_{<j}, \cdot, x_{> j})] - D_0[\phi(x_{<j}, \cdot, x_{> j})]\bigg|\bigg] > \frac{\tau}{k}.
\end{equation*}
By an averaging argument (and using the fact that $\phi$ takes values between $-1$ and $+1$), we get that with probability at least $\tau/(4 \cdot k)$ over the choice of $j\sim [k]$, $x_{< j} \sim D^{j-1}$ and $x_{> j} \sim D_0^{k-j}$, we have that
\begin{equation*}\label{eq:after_whp_switch}
\bigg|D[\phi(x_{<j}, \cdot, x_{> j})] - D_0[\phi(x_{<j}, \cdot, x_{> j})]\bigg| > \frac{\tau}{2 \cdot k}.
\end{equation*}

Since $\calD$ is a $\gamma$-flat class of distributions, there exists a (fixed) distribution $\bar{D}$ over $X$ such that for every measurable event $E \subset X$, $\Pr_{x\sim D}[x \in E] \le \gamma \cdot \Pr_{x\sim \bar{D}}[x \in E]$. Thus, we can replace the unknown input distribution $D$ by the distribution $\bar{D}$ and get that, with probability at least $\tau/(4 \cdot k \cdot \gamma^{k-1})$ over the choice of $j\sim [k]$, $x_{< j} \sim \bar{D}^{j-1}$ and $x_{> j} \sim D_0^{k-j}$, we have
\begin{equation}\label{eq:after_flat}
\bigg|D[\phi(x_{<j}, \cdot, x_{> j})] - D_0[\phi(x_{<j}, \cdot, x_{> j})]\bigg| > \frac{\tau}{2 \cdot k}.
\end{equation}
We now consider the following distribution $\mathcal{P'}$ over unary SQ functions (i.e., over $[-1,+1]^X$): Independently sample $\phi$ from $\calP$, $j$ uniformly from $[k]$, $x_{< j} \sim \bar{D}^{j-1}$ and $x_{> j} \sim D_0^{k-j}$, and output the (unary) function $\phi'(x) = \phi(x_{<j}, x, x_{> j})$. Then, for every $D\in \D$, we have that with probability at least $\frac{1}{d }\cdot \frac{\tau}{4k} \cdot \frac{1}{\gamma^{k-1}}$ over the choice of $\phi'$ from $\calP'$, we have that $|D[\phi'] - D_0[\phi']| > \tau/(2 \cdot k)$. Thus, by Fact~\ref{fa:rcvr} $$\RSD_{\kappa_1}^{(1)}\left(\calB(\calD,D_0),\frac{\tau}{2 \cdot k}\right) \le \frac{4 d \cdot \gamma^{k-1} \cdot k}{\tau}.$$
\end{proof}

Lemma \ref{lem:k-wise-flat-decision} together with the characterization in Theorem \ref{thm:random-algorithm2queries} imply the following upper bound on the SQ complexity of a decision problem in terms of its $k$-wise SQ complexity.
\begin{theorem}\label{thm:flat-decision-reduction}
Let $\gamma \geq 1$, $\tau > 0$ and $k \in \mathbb{N}$. Let $X$ be a domain, $\calD$ be a $\gamma$-flat class of distributions over $X$ and $D_0$ be any distribution over $X$. If there exists an algorithm that, with probability at least $2/3$ solves $\calB(\D,D_0)$ using $t$ queries to $\STAT^{(k)}_D(\tau)$, then for every $\delta>0$, there exists an algorithm that, with probability at least $1-\delta$ solves $\calB(\D,D_0)$ using $t \cdot 12k \cdot \gamma^{k-1} \cdot \ln(1/\delta) /\tau$ queries to $\STAT^{(1)}_D(\tau/(4k))$.
\end{theorem}

\subsection{General problems}
We now define the general class of problems over sets of distributions and a notion of statistical dimension for these types of problems.
\begin{defn}[Search problems]
A search problem $\calZ$ over a class $\calD$ of distributions and a set $\calF$ of solutions is a mapping $\calZ: \calD \to 2^{\calF} \setminus \{\emptyset\}$, where $2^{\calF}$ denotes the set of all subsets of $\calF$. Specifically, for every distribution $D \in \calD$, $\calZ(D) \subseteq \calF$ is the (non-empty) set of valid solutions for $D$. For a solution $f \in \calF$, we denote by $\calZ_f$ the set of all distributions for which $f$ is a valid solution.
\end{defn}

\begin{defn}[Statistical dimension for search problems \cite{Feldman:16sqd}]\label{def:search_SD}
For $\tau > 0$, $k \in \mathbb{N}$, a domain $X$ and a search problem $\calZ$ over a class of distributions $\calD$ over $X$ and a set of solutions $\calF$, we define the \emph{$k$-wise statistical dimension} with $\kappa_1$-discrimination $\tau$ of $\calZ$ as
\begin{equation*}
\SD^{(k)}_{\kappa_1}(\calZ,\tau) \doteq \sup_{D_0 \in S^X} \inf_{f \in \calF} \RSD^{(k)}_{\kappa_1}(\calB(\calD \setminus \calZ_f, D_0), \tau),
\end{equation*}
where $S^X$ denotes the set of all probability distributions over $X$.
\end{defn}

Lemma~\ref{lem:search_lb} lower-bounds the deterministic $k$-wise SQ complexity of a search problem in terms of its ($k$-wise) statistical dimension.
\begin{theorem}[\cite{Feldman:16sqd}]\label{lem:search_lb}
Let $\calZ$ be a search problem, $\tau > 0$ and $k \in \mathbb{N}$. The deterministic $k$-wise SQ complexity of solving $\calZ$ with access to $\STAT^{(k)}(\tau)$ is at least $\SD^{(k)}_{\kappa_1}(\calZ,\tau)$.
\end{theorem}

The following theorem from \cite{Feldman:16sqd} gives an upper bound on the SQ complexity of a search problem in terms of its statistical dimension. It relies on the multiplicative weights update method to reconstruct the unknown distribution sufficiently well for solving the problem. The use of this algorithm introduces dependence on $\KL$-radius of $\D$. Namely, we define
$$\KLR(\D) \doteq \inf_{\bar D \in S^X } \sup_{D\in \D} \KL(D\|\bar D), $$
where $\KL(\cdot \| \cdot)$ denotes the KL-divergence.
\begin{theorem}
[\cite{Feldman:16sqd}]\label{lem:search_ub}
Let $\calZ$ be a search problem, $\tau, \delta > 0$ and $k \in \mathbb{N}$. There is a randomized $k$-wise SQ algorithm that solves $\calZ$ with success probability $1-\delta$ using
$$ O\bigg(\SD^{(k)}_{\kappa_1}(\calZ,\tau) \cdot \frac{\KLR(\calD)}{\tau^2} \cdot \log\bigg( \frac{\KLR(\calD)}{\tau \cdot \delta}\bigg) \bigg)$$
queries to $\STAT^{(k)}(\tau/3)$.
\end{theorem}

Note that $\KL$-divergence between two distributions is upper-bounded (and is usually much smaller) than the max-divergence we used in the definition of $\gamma$-flatness. Specifically, if $\D$ is $\gamma$-flat  then $\KLR(\D) \leq \ln \gamma$.
We are now ready to prove Theorem~\ref{thm:flat} which we restate here for convenience.
\begin{reptheorem}{thm:flat}[restated]
Let $\gamma \geq 1$, $\tau > 0$ and $k$ be any positive integer. Let $X$ be a domain and $\calD$ be a $\gamma$-flat class of distributions over $X$. There exists a randomized algorithm that given any $\delta > 0$ and a $k$-ary function $\phi: X^k \to [-1,1]$, estimates $D^k[\phi]$ within $\tau$  for every (unknown) $D \in \calD$ with success probability at least $1-\delta$ using $$\tilde{O}\bigg( \frac{\gamma^{k-1} \cdot k^3}{\tau^3} \cdot \log (1/\delta)\bigg)$$
queries to $\STAT_D^{(1)}(\tau/(6 \cdot k))$.\end{reptheorem}
\begin{proof}
We first observe that the task of estimating $D^k[\phi]$ up to additive $\tau$ can be viewed as a search problem $\calZ$ over the set $\calD$ of distributions and over the class $\calF$ of solutions that corresponds to the interval $[-1,+1]$. Next, observe that one can easily estimate $D^k[\phi]$ up to additive $\tau$ using a single query to $\STAT^{(k)}_D(\tau)$. Lemma~\ref{lem:search_lb} implies that $\SD^{(k)}_{\kappa_1}(\calZ,\tau) = 1$. By Definition~\ref{def:search_SD}, for every $D_1 \in S^X$, there exists $f \in \calF$, such that $\RSD^{(k)}_{\kappa_1}(\calB(\calD \setminus \calZ_f, D_1), \tau) = 1$. By Lemma \ref{lem:k-wise-flat-decision},
 $$\RSD_{\kappa_1}^{(1)}\left(\calB(\calD \setminus \calZ_f, D_1),\frac{\tau}{2 \cdot k}\right) \le \frac{4 \cdot \gamma^{k-1} \cdot k}{\tau}.$$

Thus, Fact~\ref{fa:rcvr} and Definition~\ref{def:search_SD} imply that
$$ \SD_{\kappa_1}^{(1)}(\calZ,\frac{\tau}{2 \cdot k}) \le \frac{4 \cdot \gamma^{k-1} \cdot k}{\tau}.$$
Applying Lemma~\ref{lem:search_ub}, we conclude that there exists a randomized unary SQ algorithm that solves $\calZ$ with probability at least $1-\delta$ using at most $$O\bigg(\gamma^{k-1} \cdot k^3 \cdot \frac{\KLR(\mathcal{D})}{\tau^3} \cdot \log\bigg( \frac{k \cdot \KLR(\calD)}{\tau \cdot \delta}\bigg)\bigg)$$
queries to $\STAT^{(1)}(\tau/(6 \cdot k))$. This -- along with the fact that $\KLR(\calD) \le \ln(\gamma)$ whenever $\calD$ is a $\gamma$-flat set of distributions -- concludes the proof of Theorem~\ref{thm:flat}.
\end{proof}

\paragraph{Other divergences:} While the max-divergence that we used for measuring flatness suffices for the applications we give in this paper (and is relatively simple), it might be too conservative in other problems. For example, such divergence is infinite even for two Gaussian distributions with the same standard deviation but different means. A simple way to obtain a more robust version of our reduction is to use approximate max-divergence. For $\delta \in [0,1)$ it is defined as: $$\Div_\infty^\delta(D\|\bar D) \doteq \ln \sup_{E \subseteq X} \frac{\Pr_{x\sim D}[x\in E]-\delta}{\Pr_{x\sim \bar D}[x\in E]} .$$ Note that $\Div_\infty^0(D\|\bar D) = \Div_\infty(D\|\bar D)$. Similarly, we can define a radius of $\D$ in this divergence $$R_\infty^\delta(\D) \doteq \inf_{\bar D \in S^X} \sup_{D\in \D} \Div_\infty^\delta(D\|\bar D) .$$

Now, it is easy to see that, if $\Div_\infty^\delta(D\|\bar D) \leq r$ then $\Div_\infty^{k\delta}(D^k\|\bar D^k) \leq kr$. This means that if in the proof of Lemma \ref{lem:k-wise-flat-decision} we use the condition
$R_\infty^{\tau/(8k^2)}(\D) \leq \ln \gamma$ instead of $\gamma$-flatness then we will obtain that the event in Equation \eqref{eq:after_flat} holds with probability at least $$ \left( \frac{\tau}{4k} - (k-1) \cdot  \frac{\tau}{8k^2}\right) / \gamma^{k-1} \geq \frac{\tau}{\gamma^{k-1} \cdot 8k}$$ over the same random choices.

This implies the following generalization of Theorem \ref{thm:flat}.
\begin{theorem}\label{thm:flat-approx}
Let $\tau > 0$ and $k$ be any positive integer. Let $\calD$ be a class of distributions over a domain $X$ and $\gamma = \exp(R_\infty^{\tau/(8k^2)}(\D))$. There exists a randomized algorithm that given any $\delta > 0$ and a $k$-ary function $\phi: X^k \to [-1,1]$, estimates $D^k[\phi]$ within $\tau$  for every (unknown) $D \in \calD$ with success probability at least $1-\delta$ using $$\tilde{O}\bigg( \frac{\gamma^{k-1} \cdot k^3 \cdot \KLR(\calD)}{\tau^3} \cdot \log (1/\delta)\bigg)$$
queries to $\STAT_D^{(1)}(\tau/(6 \cdot k))$.\end{theorem}

An alternative approach is to use Renyi divergence of order $\alpha > 1$ defined as follows:
 $$\Div_\alpha(D\|\bar D) \doteq \frac{1}{1-\alpha} \cdot \ln \left (\E_{y \sim D} \left[\left(\frac{\Pr_{x\sim D}[x=y]}{\Pr_{x\sim \bar D}[x=y]}\right)^{\alpha-1}\right] \right).$$ The corresponding radius is defined as $$R_\alpha(\D) \doteq \inf_{\bar D \in S^X} \sup_{D\in \D} \Div_\alpha(D\|\bar D) .$$

To use it in our application we need the standard property of the Renyi divergence for product distributions $\Div_\alpha(D^k\|\bar D^k) = k \cdot \Div_\alpha(D\|\bar D)$ and also the following simple lemma from \cite[Lemma 1]{MansourMR09}:
\begin{lem}
For $\alpha > 1$, any two distributions $D,\bar D$ over $X$ and an event $E\subseteq X$:
$$\Pr_{x\sim D}[x\in E] \leq \left(\exp(\Div_\alpha(D\|\bar D) ) \cdot \Pr_{x\sim \bar D}[x\in E]\right)^{\frac{\alpha-1}{\alpha}}. $$
\end{lem}
We will need the inverted version of this lemma:
$$\Pr_{x\sim \bar D}[x\in E] \geq \frac{\left(\Pr_{x\sim D}[x\in E]\right)^{\frac{\alpha}{\alpha-1}}}{\exp(\Div_\alpha(D\|\bar D) )}. $$
Applying this in the proof of Lemma \ref{lem:k-wise-flat-decision} for $\gamma = \exp(R_\alpha(\D))$, we obtain that the event in Equation \eqref{eq:after_flat} holds with probability at least $$ \left(\frac{\tau}{4k}\right)^{\frac{\alpha}{\alpha-1}}/ \gamma^{k-1}. $$
This gives the following generalization of Theorem \ref{thm:flat}.
\begin{theorem}\label{thm:flat-approx}
Let $\tau > 0,\alpha >1$ and $k$ be any positive integer. Let $\calD$ be a class of distributions over a domain $X$ and $\gamma = \exp(R_\alpha(\D))$. There exists a randomized algorithm that given any $\delta > 0$ and a $k$-ary function $\phi: X^k \to [-1,1]$, estimates $D^k[\phi]$ within $\tau$  for every (unknown) $D \in \calD$ with success probability at least $1-\delta$ using $$\tilde{O}\bigg( \gamma^{k-1} \cdot \left( \frac{k}{\tau}\right)^{2+\frac{\alpha}{\alpha-1}} \cdot \log (1/\delta)\bigg)$$
queries to $\STAT_D^{(1)}(\tau/(6 \cdot k))$.\end{theorem}

\subsection{Applications to solving CSPs and learning DNF}
\label{sec:lower-bounds}
We now give some examples of the application of our reduction to obtain lower bounds against $k$-wise SQ algorithms. Our applications for stochastic constraint satisfaction problems (CSPs) and DNF learning. We start with the definition of a stochastic CSP with a {\em planted solution} which is a pseudo-random generator based on Goldreich's proposed one-way function \cite{goldreich2000candidate}.
\begin{definition}
Let $t \in \mathbb{N}$ and $P: \{\pm 1\}^t \to \{\pm 1\}$ be a fixed predicate. We are given access to samples from a distribution $P_{\sigma}$, corresponding to a (``planted'') assignment $\sigma \in \{\pm 1 \}^n$. A sample from this distribution is a uniform-random  $t$-tuple $(i_1, \dots, i_t)$ of distinct variable indices along with the value $P(\sigma_{i_1}, \dots, \sigma_{i_t})$. The goal is to recover the assignment $\sigma$ when given $m$ independent samples from $P_{\sigma}$. A (potentially) easier problem is to distinguish any such planted distribution from the distribution $U_t$ in which the value is an independent uniform-random coin flip (instead of $P(\sigma_{i_1}, \dots, \sigma_{i_t})$).
\end{definition}
We say that a predicate $P: \{\pm 1\}^t \to \{\pm 1\}$ has complexity $r$ if $r$ is the degree of the lowest-degree non-zero Fourier coefficient of $P$. It can be as large as $t$ (for the parity function).
A lower bound on the (unary) SQ complexity of solving such CSPs was shown by \cite{FeldmanPV:13} (their result is for the stronger $\VSTAT$ oracle but here we state the version for the $\STAT$ oracle).
\begin{theorem}[\cite{FeldmanPV:13}]\label{thm:FPV_csp}
Let $t, q \in \mathbb{N}$ and $P: \{\pm 1\}^t \to \{\pm 1\}$ be a fixed predicate of complexity $r$. Then for any $q >0$, any  algorithm that, given access to a distribution $D \in \{P_\sigma\ |\  \sigma  \in \{\pm 1 \}^n\} \cup \{U_t\}$  decides correctly whether $D = P_\sigma$ or $D=U_t$ with probability at least $2/3$ needs $q/2^{O(t)}$ queries to $\STAT^{(1)}_D\left(\left(\frac{\log q}{n}\right)^{r/2}\right)$.
\end{theorem}

The set of input distributions in this problem is $2$-flat relative to $U_t$ and it is one-to-many decision problem. Hence Theorem \ref{thm:flat-decision-reduction} implies\footnote{We can also get essentially the same result by applying the simulation of a $k$-wise SQ using unary SQs from Theorem \ref{thm:flat}.} the following lower bound for $k$-wise SQ algorithms.
\begin{theorem}\label{thm:csp-k-wise}
Let $t \in \mathbb{N}$ and $P: \{\pm 1\}^t \to \{\pm 1\}$ be a fixed predicate of complexity $r$. Then for any $\alpha >0$, any  algorithm that, given access to a distribution $D \in \{P_\sigma\ |\  \sigma  \in \{\pm 1 \}^n\} \cup \{U_t\}$  decides correctly whether $D = P_\sigma$ or $D=U_t$ with probability at least $2/3$ needs $2^{n^{1-\alpha} -O(t)}$ queries to $\STAT^{(n^{1-\alpha})}_D\left((2/n^{\alpha})^{r/2}  \cdot n^{1-\alpha}/4\right)$.
\end{theorem}
\begin{proof}
Let $\A$ be a $k$-wise SQ algorithm using $q'$ queries to $\STAT^{(n^{1-\alpha})}_D\left((2/n^{\alpha})^{r/2} \cdot n^{1-\alpha}/6\right)$ which solves the problem with success probability $2/3$.
We let $k=n^{1-\alpha}$ and apply Theorem \ref{thm:flat-decision-reduction} to obtain an algorithm that uses unary SQs and solves the problem with success probability $2/3$. This algorithm uses $q_0 = q' \cdot 2^{n^{1-\alpha}} \cdot n^{O(r)}$ queries to $\STAT^{(1)}_D\left((2/n^{\alpha})^{r/2}\right)$. Now choosing $q = 2^{2n^{1-\alpha}}$ we get that $\left(\frac{\log q}{n}\right)^{r/2} \leq (2/n^{\alpha})^{r/2}$. This means that $q_0 \geq q/2^{O(t)} = 2^{2n^{1-\alpha}- O(t)}$.
Hence $q' = 2^{2n^{1-\alpha}-O(t) - n^{1-\alpha}- O(r)} = 2^{n^{1-\alpha} -O(t)}$.
\end{proof}
Similar lower bounds can be obtained for other problems considered in \cite{FeldmanPV:13}, namely, planted satisfiability and $t$-SAT refutation.

To obtain a lower bound for learning DNF formulas we can use a simple reduction from the Goldreich's PRG defined above to learning DNF formulas of polynomial size. It is based on ideas implicit in the reduction from $t$-SAT refutation to DNF learning from \cite{DanielyS16}.
\begin{lemma}\label{lem:reduce-dnf}
$P: \{\pm 1\}^t \to \{\pm 1\}$ be a fixed predicate. There exists a mapping $M$ from $t$-tuples of indices in $[n]$ to $\zo^{tn}$ such that for every $\sigma \in \{\pm 1 \}^n$ there exists a DNF formula $f_\sigma$ of size $2^t$ satisfying $P(\sigma_{i_1}, \dots, \sigma_{i_t}) = f_\sigma(M(i_1, \dots, i_t))$.
\end{lemma}
\begin{proof}
The mapping $M$ maps $(i_1, \dots, i_t)$ to the concatenation of the indicator vectors of each of the indices. Namely, for $j \in [t]$ and $\ell \in [n]$, $M(i_1, \dots, i_t)_{j,\ell} = 1$ if and only if $i_j = \ell$, where we use the double index $j,\ell$ to refer to element $n (j-1) + \ell$ of the vector. Let $v_{j,\ell}$ denote the variable with the index $j,\ell$. Let $\sigma$ be any assignment and we denote by $z_j^\sigma$ the $j$-th variable of our predicate $P$ when the assignment is equal to $\sigma$. We first observe that $z_j^\sigma \equiv \bigwedge_{\ell \in [n], \sigma_\ell = 0} \bar{v}_{j,\ell}$. This is true since, by definition, the value of the $j$-th variable of our predicate is $\sigma_{i_j}$. This value is $1$ if and only if $i_j \not\in \{\ell \in [n]\ | \ \sigma_\ell = 0\}$. This is equivalent to $v_{j,\ell}$ being equal to $0$ for all $\ell \in [n]$ such that  $\sigma_\ell = 0$. Analogously, $\bar{z}^\sigma_j \equiv \bigwedge_{\ell \in [n], \sigma_\ell = 1} \bar{v}_{j,\ell}$. This implies that any conjunction of variables $z_1^\sigma,\bar{z}^\sigma_1,\ldots,z^\sigma_t,\bar{z}^\sigma_t$ can be expressed as a conjunction over variables $\bar{v}_{j,\ell}$. Any predicate $P$ can be expressed as a disjunction of at most $2^t$ conjunctions and hence there exists a DNF formula $f_\sigma$ of size at most $2^t$ whose value on $M(i_1, \dots, i_t)$ is equal to $P(\sigma_{i_1}, \dots, \sigma_{i_t})$
\end{proof}

This reduction implies that by converting a sample $((i_1, \dots, i_t),b)$ to a sample $(M(i_1, \dots, i_t),b)$ we can transform the Goldreich's PRG problem into a problem in which our goal is to distinguish examples of some DNF formula $f_\sigma$ from randomly labeled examples. Naturally, an algorithm that can learn DNF formulas can output a hypothesis which predicts the label (with some non-trivial accuracy), whereas such hypothesis cannot exist for predicting random labels. Hence known SQ lower bounds on planted CSPs \cite{FeldmanPV:13} immediately imply lower bounds for learning DNF. Further, by applying Lemma \ref{lem:reduce-dnf} together with Thm.~\ref{thm:csp-k-wise} for $t=r=\log n$ we obtain the first lower bounds for learning DNF against $n^{1-\alpha}$-wise SQ algorithms.
\begin{theorem}\label{thm:dnf-k-wise}
For any constant (independent of $n$) $\alpha >0$, there exists a constant $\beta>0$ such that
 any  algorithm that PAC learns DNF formulas of size $n$ with error $<1/2 - n^{- \beta \log n}$ and success probability at least $2/3$ needs at least $2^{n^{1-\alpha}}$ queries to $\STAT^{(n^{1-\alpha})}_D(n^{- \beta \log n})$.
\end{theorem}
We remark that this is a lower bound for PAC learning polynomial size DNF formulas with respect to some fixed (albeit non-uniform) distribution over $\zo^n$. The approach for relating $k$-wise SQ complexity to unary SQ complexity given in \cite{blum2003noise} applies to this setting. Yet, in their proof the tolerance needed for the unary SQ algorithm is $\tau/2^k$ and therefore it would not give a non-trivial lower bounds beyond $k=O(\log n)$.


\section{Reduction for low-communication queries}\label{sec:sq_cc}
In this section, we prove  Theorem~\ref{thm:sq_and_cc} using a recent result of Steinhardt, Valiant and Wager \cite{SteinhardtVW16}.
Their result can be seen giving a SQ algorithm that simulates a communication protocol between $n$ parties. Each party is holding a sample drawn i.i.d.~from distribution $D$ and broadcasts at most $b$ bits about its sample (to all the other parties). The bits can be sent over multiple rounds. This is essentially the standard model of multi-party communication complexity (\eg \cite{KN97}) but with the goal of solving some problem about the unknown distribution $D$ rather than computing a specific function of the inputs. Alternatively, one can also see this model as a single algorithm that extracts at most $b$-bits of information about each random sample from $D$ and is allowed to extract the bits in an arbitrary order (generalizing the $b$-bit sampling model that we discuss in Section \ref{subsec:RFA} and in which $b$-bits are extracted from each sample at once). We refer to this model simply as algorithms that extract at most $b$ bits per sample.

\begin{theorem}[\cite{SteinhardtVW16}]\label{thm:SVW}
Let $\calA$ be an algorithm that uses $n$ samples drawn i.i.d.~from a distribution $D$ and extracts at most $b$ bits per sample. Then, for every $\beta > 0$, there is an algorithm $\calB$ that makes at most $2 \cdot b \cdot n$ queries to $\STAT^{(1)}_D(\beta/(2^{b+1} \cdot k))$ and the output distributions of $\calA$ and $\calB$ are within total variation distance $\beta$.
\end{theorem}

We will use this simulation to estimate the expectation of $k$-wise functions that have low communication complexity.
Specifically, we recall the following standard model of public-coin randomized $k$-party communication complexity.
\begin{defn}
For a function $\phi:X^k \to \{\pm 1\}$ we say that $\phi$ has a $k$-party public-coin randomized communication complexity of at most $b$ bits per party with success probability $1-\delta$ if there exist a protocol satisfying the following conditions. Each of the parties is given $x_i \in X$ and access to shared random bits. In each round one of the parties can compute one or more bits using its input, random bits and all the previous communication and then broadcast it to all the other parties. In the last round one of the parties computes a bit that is the output of the protocol. Each of the parties communicates at most $b$ bits in total. For every $x_1,\ldots,x_k\in X$, with probability at least $1-\delta$ over the choice of the random bits the output of the protocol is equal to $\phi(x_1,\ldots,x_k)$.
\end{defn}

We are now ready to prove Theorem~\ref{thm:sq_and_cc} which we restate here for convenience.
\begin{reptheorem}{thm:sq_and_cc}[restated]
Let $\phi:X^k \to \{\pm 1\}$ be a function, and assume that $\phi$ has $k$-party public-coin randomized communication complexity of $b$ bits per party with success probability $2/3$. Then, there exists a randomized algorithm that, with probability at least $1-\delta$, estimates $\Ex_{x \sim D^k}[\phi(x)]$ within $\tau$ using $O(b \cdot k \cdot \log(1/\delta)/\tau^2)$ queries to $\STAT^{(1)}_D(\tau')$ for some $\tau' = \tau^{O(b)}/k$.
\end{reptheorem}
\begin{proof}
We first amplify the success probability of the protocol for computing $\phi$ to $\delta'\doteq \tau/8$ using the majority vote of $O(\log(1/\delta'))$ repetitions. By Yao's minimax theorem \cite{Yao:1977} there exists a deterministic protocol $\Pi'$ that succeeds with probability at least $1-\delta'$ for $(x_1,\dots,x_k) \sim D^k$. Applying Theorem~\ref{thm:SVW}, we obtain a unary SQ algorithm $\A$ whose output is within total variation distance at most $\beta \doteq \tau/8$ from $\Pi'(x_1, \dots, x_k)$ (and we can assume that the output of $\A$ is in $\{\pm 1\}$). Therefore:
$$|\E[\A] - D^k[\phi]| \leq |\E[\A] - \E_{D^k} [\Pi'(x_1, \dots, x_k)]| + |\E_{D^k} [\Pi'(x_1, \dots, x_k)] - D^k [\phi]| \leq  \frac{2\tau}{8} + \frac{2\tau}{8} = \frac{\tau}{2}.$$
Repeating $\A$ $O(\log(1/\delta)/\tau^2)$ times and taking the mean, we get an estimate of $D^k[\phi]$ within $\tau$ with probability at least $1-\delta$. This algorithm uses $O(b \cdot k \cdot \log(1/\delta)/\tau^2)$ queries to $\STAT^{(1)}_D(\tau')$ for $\tau' = \frac{\tau}{8}/(2^{O(\log(8/\tau) \cdot b)} \cdot k) = \tau^{O(b)}/k$.
\end{proof}

The collision probability for a distribution $D$ is defined as $\Pr_{(x_1, x_2) \sim D^2}[x_1 = x_2]$. This corresponds to $\phi(x_1,x_2)$ being the Equality function which, as is well-known, has randomized $2$-party communication complexity of $O(1)$ bits per party with success probability $2/3$ (see, e.g., \cite{KN97}). Applying Theorem~\ref{thm:sq_and_cc} with $k=2$ we get the following corollary.
\begin{corollary}
For any $\tau, \delta >0$, there is a SQ algorithm that estimates the collision probability of an unknown distribution $D$ within $\tau$ with success probability $1-\delta$ using $O(\log(1/\delta)/\tau^2)$ queries to $\STAT^{(1)}_{D}(\tau^{O(1)})$.
\end{corollary} 

\section{Corollaries for other models}
\label{sec:apps}

\subsection{$k$-local differential privacy}\label{subsec:local_DP}
We start by formally defining the $k$-wise version of the \emph{local differentially privacy} model from \cite{kasiviswanathan2011can}.

\begin{defn}[$k$-local randomizer]
A $k$-local $\eps$-differentially private (DP) randomizer is a randomized map $R: X^k \to W$ such that for all $u, u' \in X^k$ and all $w \in W$, we have that $\Pr[R(u) = w] \le e^{\epsilon} \cdot \Pr[R(u') = w]$ where the probabilities are taken over the coins of $R$.
\end{defn}

The following definition gives a $k$-wise generalization of the local randomizer (LR) oracle which was used in \cite{kasiviswanathan2011can}.

\begin{defn}[$k$-local Randomizer Oracle]
Let $z = (z_1,\dots,z_n) \in X^n$ be a database. A $k$-LR oracle $\LR_z(\cdot,\cdot)$ gets a $k$-tuple of indices $\bar i \in [n]^k$ and a $k$-local $\eps$-DP randomizer as inputs, and outputs an element $w \in W$ which is sampled from the distribution $R(z_{i_1},\ldots,z_{i_k})$.
\end{defn}

We are now ready to give the definition of $k$-local differential privacy.
\begin{defn}[$k$-local differentially private algorithm]
A $k$-local $\epsilon$-differentially private algorithm is an algorithm that accesses a database $z \in X^n$ via a $k$-LR oracle $\LR_z$ with the restriction that for all $i \in [n]$, if $\LR_z(\bar i_1,R_1), \dots, \LR_z(\bar i_t,R_t)$ are the algorithm's invocations of $\LR_z$ on $k$-tuples of indices that include index $i$, where for each $j \in [t]$ $R_j$ is a $k$-local $\epsilon_j$-DP randomizer, then $\epsilon_1 + \dots + \epsilon_t \le \epsilon$.
\end{defn}

The following two theorems -- which follow from Theorem 5.7 and Lemma 5.8 of \cite{kasiviswanathan2011can} -- show that $k$-local differentially private algorithms are equivalent (up to polynomial factors) to $k$-wise statistical query algorithms.

\begin{theorem}\label{thm:local_DP_first}
Let $\calA_{SQ}$ be a $k$-wise SQ algorithm that makes at most $t$ queries to $\STAT^{(k)}_D(\tau)$. Then, for every $\beta>0$, there exists a $k$-local $\epsilon$-DP algorithm $\calA_{DP}$ such that if the database $z$ has $n \geq n_0=O(k \cdot t \cdot \log(t/\beta)/(\epsilon^2 \cdot \tau^2))$ entries sampled i.i.d.~from the distribution $D$, then $\calA_{DP}$ makes $n_0/k$ queries and the total variation between $\calA_{DP}$'s and $\calA_{SQ}$'s output distributions is at most $\beta$.
\end{theorem}

\begin{theorem}\label{thm:local_DP_sec}
Let $z \in X^n$ be a database with entries drawn i.i.d.~from a distribution $D$. For every $k$-local  $\epsilon$-DP algorithm $\calA_{DP}$ making $t$ queries to $\LR_z$ and $\beta >0$, there exists a $k$-wise statistical query algorithm $\calA_{SQ}$ that in expectation makes $O(t \cdot e^{\epsilon})$ queries to $\STAT^{(k)}_D(\tau)$ for $\tau = \Theta(\beta/(e^{2\epsilon} \cdot t))$ such that the total variation between $\calA_{SQ}$'s and $\calA_{DP}$'s output distributions is at most $\beta$.
\end{theorem}

By combining Theorem~\ref{thm:k_wise_sep}, Theorem~\ref{thm:local_DP_first} and Theorem~\ref{thm:local_DP_sec} we then obtain the following corollary.
\begin{corollary}
\label{cor:local_DP}
For every positive integer $k$ and any prime number $p$, there is a concept class $\calC$ of Boolean functions defined over a domain of size $p^{k+1}$ for which there exists a $(k+1)$-local 1-DP distribution-independent PAC learning algorithm using a database consisting of $\widetilde{O}_k(\log{p})$ i.i.d.~samples, whereas any $k$-local 1-DP distribution-independent PAC learning algorithm requires $\Omega_k(p^{1/4})$ samples.
\end{corollary}

The reduction in Theorem \ref{thm:flat} then implies that for $\gamma$-flat classes of distributions a $k$-local DP algorithm can be simulated by a $1$-local DP algorithm with an overhead that is linear in $\gamma^{k-1}$ and polynomial in other parameters.
\begin{theorem}\label{thm:flat-dp}
Let $\gamma \geq 1$,  $k$ be any positive integer. Let $X$ be a domain and $\calD$ a $\gamma$-flat class of distributions over $X$. Let $z \in X^n$ be a database with entries drawn i.i.d. from a distribution $D \in \calD$.  For every $k$-local $\epsilon$-DP algorithm $\calA$ making $t$ queries to a $k$-LR oracle $\LR_z$ and $\beta>0$, there exists a $1$-local $\eps$-DP algorithm $\calB$ such that if $n \geq n_0=
\tilde{O}\bigg(\frac{\gamma^{k-1}\cdot t^6 \cdot k^6\cdot e^{11\eps} }{\beta^3\eps^2} \bigg)$
then  for every $D\in \D$, $\calB$ makes $n_0/k$ queries to $1$-LR oracle $\LR'_z$ and the total variation distance between $\calB$'s and $\calA$'s output distributions is at most $\beta$.
\end{theorem}
The reduction from Theorem \ref{thm:sq_and_cc} can be translated to this model analogously.

\subsection{$k$-wise $b$-bit sampling model}\label{subsec:RFA}
For an integer $b>0$, a $b$-bit sampling oracle $\COMM_D(b)$ is defined as follows: Given any function $\phi: X \to \zo^b$, $\COMM_D(b)$  returns $\phi(x)$ for $x$ drawn randomly and independently from $D$, where $D$ is the unknown input distribution. This oracle was first studied by Ben-David and Dichterman \cite{Ben-DavidD98} as a {\em weak Restricted Focus of Attention} model. They showed that algorithms in this model can be simulated efficiently using statistical queries and vice versa. Lower bounds against algorithms that use such an oracle have been studied in \cite{FeldmanGRVX:12,FeldmanPV:13}. More recently, motivated by communication constraints in distributed systems, the sample complexity of several basic problems in statistical estimation has been studied in this and related models \cite{ZhangDJW13,SteinhardtD15,SteinhardtVW16}. These works also study the natural $k$-wise generalization of this model. Specifically, $\COMM^{(k)}_D(b)$ is the oracle that given any function $\phi: X^k \to \zo^b$, returns $\phi(x)$ for $x$ drawn randomly and independently from $D^k$.

The following two theorems -- which follow from Theorem 5.2 in \cite{Ben-DavidD98} and Proposition 3 in \cite{SteinhardtVW16} (that strengthens a similar result in \cite{Ben-DavidD98}) -- show that $k$-wise algorithms in the $b$-bit sampling model are equivalent (up to polynomial and $2^b$ factors) to $k$-wise statistical query algorithms.

\begin{theorem}\label{thm:RFA_first}
Let $\calA_{SQ}$ be a $k$-wise SQ algorithm that makes at most $t$ Boolean queries to $\STAT^{(k)}_D(\tau)$. Then, for every $\beta>0$, there exists a $k$-wise $1$-bit sampling algorithm $\calA_{1\text{-bit}}$ that uses $O(\frac{t}{\tau^2} \cdot \log(t/\beta))$ queries to  $\COMM^{(k)}_D(b)$ and the total variation distance between $\calA_{SQ}$'s and $\calA_{1\text{-bit}}$'s output distributions is at most $\beta$.
\end{theorem}

\begin{theorem}\label{thm:RFA_sec}
Let $\calA_{b \text{-bit}}$ be a $k$-wise $b$-bit sampling algorithm that makes at most $t$ queries to $\COMM^{(k)}_D(b)$. Then, for every $\beta>0$, there exists a $k$-wise SQ algorithm $\calA_{SQ}$ that makes $2bt$ queries to $\STAT^{(k)}_D(\beta/(2^{b+1} t))$ and the total variation distance between $\calA_{SQ}$'s and $\calA_{b\text{-bit}}$'s output distributions is at most $\beta$.
\end{theorem}

Feldman \etal \cite{FeldmanGRVX:12} give a tighter correspondence between the $\COMM$ oracle and the slightly stronger $\VSTAT$ oracle. Their simulations can be extended to the $k$-wise case in a similar way.

The following corollary now follows by combining Theorem~\ref{thm:k_wise_sep}, Theorem~\ref{thm:RFA_first} and Theorem~\ref{thm:RFA_sec}.

\begin{corollary}\label{cor:RFA-separation}
Let $b = O(1)$. For every positive integer $k$ and any prime number $p$, there is a concept class $\calC$ of Boolean functions defined over a domain of size $p^{k+1}$ for which there exists a $(k+1)$-wise $b$-bit sampling distribution-independent PAC learning algorithm making $\widetilde{O}_k(\log{p})$ queries, whereas any $k$-wise $b$-bit sampling distribution-independent PAC learning algorithm requires $\widetilde{\Omega}_k(p^{1/{12}})$ queries.
\end{corollary}

The reduction in Theorem \ref{thm:flat} then implies that for $\gamma$-flat classes of distributions a $k$-wise $1$-bit sampling algorithm can be simulated by a $1$-wise $1$-bit sampling algorithm.
\begin{theorem}\label{thm:flat-1-bit}
Let $\gamma \geq 1$,  $k$ be any positive integer. Let $X$ be a domain and $\calD$ a $\gamma$-flat class of distributions over $X$. For every algorithm $\calA$ making $t$ queries to $\COMM^{(k)}_D(1)$ and every $\beta > 0$, there exists a 1-bit sampling algorithm $\calB$ that for every $D\in \D$, uses $\tilde{O}\bigg(\frac{\gamma^{k-1}\cdot t^6 \cdot k^5 }{\beta^3} \bigg)$ queries to $\COMM_D(1)$ and the total variation distance between $\calB$'s and $\calA$'s output distributions is at most $\beta$.
\end{theorem}


\bibliographystyle{alpha}
\bibliography{refs,vf-allrefs}
\appendix

\section{Omitted proofs}

\subsection{Proof of \Cref{le:recovery}}\label{subsec:pf_rec_lem}
	In the following, we denote by $o_c(\cdot)$ and $\omega_c(\cdot)$ asymptotic functions obtained by taking the limit as the parameter $c$ goes to infinity. In particular, $o_c(1)$ can be made arbitrarily close to $0$ by letting $c$ be large enough.
	
	Let $W$ be as in the statement of \Cref{le:recovery}. To prove the lemma, it suffices to show that each bit $j$ in the binary representation of the subspace $\widehat{W}$ constructed by \Cref{alg:recovery} is equal to the corresponding bit of $W$. Henceforth, we fix $j$. We consider the two cases where bit $j$ of $W$ is equal to $1$, and where it is equal to $0$. 
	
	First, we assume that bit $j$ of $W$ is equal to $1$, and prove that in the execution of \Cref{alg:recovery}, it will be the case that $u_{i,j}/v_i \geq 1- o_c(1)$.  We can then set $c$ to be sufficiently large to ensure that $u_{i,j}/v_i \geq  (9/10)$. Note that for any positive real numbers $N$, $D$ and $\tau$ such that $\tau = o(N)$ and $\tau = o(D)$, we have that
	\begin{equation*}
	\frac{N-\tau}{D+\tau} \geq \frac{N}{D} \cdot (1-o(1)).
	\end{equation*}
	Thus, it is enough to show that the next three statements hold:
	\begin{enumerate}
		\item[(i)] $\tau = o_c( \overline{v}_i )$,
		\item[(ii)] if bit $j$ of $W$ is $1$, then $(\overline{u}_{i,j}/\overline{v}_i)	\geq 1 - o_c(1)$,
		\item[(iii)] if bit $j$ of $W$ is $1$, then $\tau = o_c( \overline{u}_{i,j} )$,
	\end{enumerate}
	where $\overline{u}_{i,j} \triangleq \Ex[\phi_{i,j}]$ and $\overline{v}_i \triangleq \Ex[\phi_i]$.
	
	To show (i) above, note that
	\begin{align*}
	\overline{v}_i &= \Pr\bigg[(b_1,\dots,b_{k+1}) = 1^{k+1}\mbox{ and }\rk(Z) = i \bigg]\\
	&\geq v_i - \tau\\
	&\geq v \cdot \tau_i - \tau\\
	&\geq \omega_c(\tau),
	\end{align*}
	where the first inequality follows from the definition of $v_i$ and the SQ guarantee, the second inequality follows from the given assumption (in the statement of \Cref{le:recovery}) that $(v_i/v) \geq \tau_i$, and the last inequality follows from the fact that since $v > \epsilon^{k+1}/2$, for every $i \in [k+1]$, we have that
	\begin{equation*}
	\tau = o_c \bigg( (v \cdot \tau_i - \tau ) \cdot (1-\tau_i/4)\bigg).
	\end{equation*}.
	
	Recall the definition of the event $E_j(Z)$ from the description of \Cref{alg:recovery}. To show (ii) above, note that
	
	\begin{align*}
	\frac{\overline{u}_{i,j}}{\overline{v}_i} &= \Pr\bigg[E_j(Z) ~ | ~ (b_1,\dots,b_{k+1}) = 1^{k+1}\mbox{ and }\rk(Z) = i \bigg]\\
	&\geq \Pr\bigg[\text{all rows of } Z \text{ belong to } W ~ | ~ (b_1,\dots,b_{k+1}) = 1^{k+1}\mbox{ and }\rk(Z) = i \bigg]\\
	&= 1- \Pr\bigg[\exists\text{ a row of } Z \text{ that } \notin W ~ | ~ (b_1,\dots,b_{k+1}) = 1^{k+1}\mbox{ and }\rk(Z) = i \bigg]\\
	&\geq 1 - (k+1) \cdot \Pr_{z \sim Q}[z \notin W]\\
	&\geq 1 - \frac{\tau_i}{4}\\
	&\geq 1- o_c(1),
	\end{align*}
	where the first inequality uses the assumption that bit $j$ in the binary representation of $W$ is $1$ and the facts that the dimension of $W$ is equal to $i$ and that we are conditioning on $\rk[Z] = i$. The second inequality follows from the union bound, the third inequality follows from the assumption given in \Cref{eq:lemma_W_assumption}, and the last inequality follows from the fact that for every $i \in [k+1]$, we have that $\tau_i = o_c(1)$.
	
	To show (iii) above, note that
	\begin{align*}
	\overline{u}_{i,j} &= \overline{v}_i \cdot \frac{\overline{u}_{i,j}}{\overline{v}_i}\\
	&\geq \omega_c(\tau) \cdot (1- o_c(1))\\
	&\geq \omega_c(\tau),
	\end{align*}
	where the first inequality follows from (i) and (ii) above.
	
	We now turn to the (slightly different) case where bit $j$ of $W$ is equal to $0$, and prove that in the execution of \Cref{alg:recovery}, we will have that $u_{i,j}/v_i = o_c(1)$. Note that for any positive real numbers $N$, $D$ and $\tau$ such that $\tau = o(D)$, we have that
	\begin{equation*}
	\frac{N+\tau}{D-\tau} \le \frac{N}{D} \cdot (1+o(1)) + o(1).
	\end{equation*}
	Thus, it is enough to use the fact that $\tau = o_c( \overline{v}_i )$ (proven in (i) above) and to show the next statement:
	\begin{enumerate}
		\item[(iv)] if bit $j$ of $W$ is $0$, then $(\overline{u}_{i,j}/\overline{v}_i)	= o_c(1)$.
	\end{enumerate}
	
	To prove (iv), note that since bit $j$ of $W$ is $0$, we have that
	\begin{align*}
	\frac{\overline{u}_{i,j}}{\overline{v}_i} &\le \Pr\bigg[\exists\text{ a row of } Z \text{ that } \notin W ~ | ~ (b_1,\dots,b_{k+1}) = 1^{k+1}\mbox{ and }\rk(Z) = i \bigg]\\
	&\le \frac{\tau_i}{4}\\
	&\le o_c(1),
	\end{align*}
	where the first inequality above follows from the assumption that bit $j$ in the binary representation of $W$ is $0$ and the facts that the dimension of $W$ is equal to $i$ and that we are conditioning on $\rk[Z] = i$. The second inequality above follows from the union bound and the assumption given in \Cref{eq:lemma_W_assumption}, and the last inequality follows from the fact that for every $i \in [k+1]$, we have that $\tau_i = o_c(1)$. As before, we choose $c$ to be sufficiently large to ensure that this last probability is smaller than $(1/10)$.

\subsection{Proof of Proposition~\ref{prop:single_ex}}
Let $a \in \mathbb{F}_p^{\ell}$. We have that:
	\begin{align*}
	\Ex_{(z,b) \sim D_0}[D_a(z,b)] &= \Ex_{(z,b) \sim D_0}\bigg[\displaystyle\prod\limits_{i=1}^k \Ex_{(z_i,b_i) \sim D_0}[D_a(z_i,b_i)\bigg]\\
	&= \displaystyle\prod\limits_{i=1}^k \Ex_{(z_i,b_i) \sim D_0}\bigg[ D_a(z_i,b_i)\bigg]\\
	&= \displaystyle\prod\limits_{i=1}^k \Ex_{(z_i,b_i) \sim D_0}\bigg[ D_a(z_i) \cdot \ind(b_i = f_a(z_i))\bigg]\\
	&= \displaystyle\prod\limits_{i=1}^k \Ex_{z_i \sim D_0}\bigg[ D_a(z_i) \cdot \Ex_{b_i \in_R \{\pm 1\}}[\ind(b_i = f_a(z_i))]\bigg]\\
	&= \frac{1}{2^k} \cdot \displaystyle\prod\limits_{i=1}^k \Ex_{z_i \sim D_0}\bigg[ D_a(z_i) \bigg]\\
	&= \frac{1}{2^k} \cdot \left(\frac{1}{p} \cdot \beta + \left(1-\frac{1}{p}\right) \cdot \alpha\right)^k.
	\end{align*}

\subsection{Proof of Proposition~\ref{prop:pair_ex}}

	Let $a, a' \in \mathbb{F}_p^{\ell}$. First, assume that $\Hyp_a = \Hyp_{a'}$, i.e., that $a = a'$. Then,
	\begin{align*}
	\Ex_{(z,b) \sim D_0}[D_a(z,b) \cdot D_{a'}(z,b)] &= \Ex_{(z,b) \sim D_0}[D_a(z,b)^2]\\
	&= \Ex_{(z,b) \sim D_0}\bigg[\displaystyle\prod\limits_{i = 1}^k D_a(z_i,b_i)^2 \bigg]\\
	&= \displaystyle\prod\limits_{i = 1}^k \Ex_{(z_i,b_i) \sim D_0} [D_a(z_i,b_i)^2]\\
	&= \displaystyle\prod\limits_{i = 1}^k \Ex_{(z_i,b_i) \sim D_0} [D_a(z_i)^2 \cdot \ind(b_i = f_a(z_i))]\\
	&= \displaystyle\prod\limits_{i = 1}^k \Ex_{z_i}\bigg[D_a(z_i)^2 \cdot \Ex_{b_i}[\ind(b_i = f_a(z_i))] \bigg]
	\end{align*}
	Thus,
	\begin{align*}
	\Ex_{(z,b) \sim D_0}[D_a(z,b) \cdot D_{a'}(z,b)] &= \frac{1}{2^k} \cdot \displaystyle\prod\limits_{i = 1}^k \Ex_{z_i}[D_a(z_i)^2]\\
	&= \frac{1}{2^k} \cdot \displaystyle\prod\limits_{i = 1}^k \left(\frac{1}{p} \cdot \beta^2 + \left(1-\frac{1}{p}\right) \cdot \alpha^2\right)\\
	&= \frac{1}{2^k} \cdot \left(\frac{1}{p} \cdot \beta^2 + \left(1-\frac{1}{p}\right) \cdot \alpha^2\right)^k.
	\end{align*}
	Now we assume that $\Hyp_a \cap \Hyp_{a'} = \emptyset$. Then,
	\begin{align*}
	\Ex_{(z,b) \sim D_0}[D_a(z,b) \cdot D_{a'}(z,b)] &= \Ex_{(z,b) \sim D_0}\bigg[\displaystyle\prod\limits_{i = 1}^k D_a(z_i,b_i) \cdot D_{a'}(z_i,b_i) \bigg]\\
	&= \displaystyle\prod\limits_{i = 1}^k \Ex_{(z_i,b_i) \sim D_0}[D_a(z_i,b_i) \cdot D_{a'}(z_i,b_i)]\\
	&= \displaystyle\prod\limits_{i = 1}^k \Ex_{(z_i,b_i) \sim D_0}[D_a(z_i) \cdot \ind(b_i = f_a(z_i)) \cdot D_{a'}(z_i) \cdot \ind(b_i = f_{a'}(z_i))]\\
	&=  \displaystyle\prod\limits_{i = 1}^k \Ex_{z_i}\bigg[D_a(z_i) \cdot D_{a'}(z_i) \cdot \ind(f_a(z_i) = f_{a'}(z_i))  \cdot \Ex_{b_i}[\ind(b_i = f_{a}(z_i))] \bigg]\\
	&=  \frac{1}{2^k} \cdot \displaystyle\prod\limits_{i = 1}^k \Ex_{z_i}\bigg[D_a(z_i) \cdot D_{a'}(z_i) \cdot \ind(f_a(z_i) = f_{a'}(z_i)) \bigg]\\
	&= \frac{1}{2^k} \cdot \displaystyle\prod\limits_{i = 1}^k \left(\alpha^2 \cdot \left(1-\frac{2}{p}\right)\right)\\
	&= \frac{1}{2^k}\cdot \left(\alpha^2 \cdot \left(1-\frac{2}{p}\right)\right)^k
	.\end{align*}
	Finally, we assume that $\Hyp_a \neq \Hyp_{a'}$ and $\Hyp_a \cap \Hyp_{a'} \neq \emptyset$. Then,
	\begin{align*}
	\Ex_{(z,b) \sim D_0}[D_a(z,b) \cdot D_{a'}(z,b)] &= \frac{1}{2^k} \cdot \displaystyle\prod\limits_{i = 1}^k \Ex_{z_i}\bigg[D_a(z_i) \cdot D_{a'}(z_i) \cdot \ind(f_a(z_i) = f_{a'}(z_i)) \bigg]\\
	&=  \frac{1}{2^k} \cdot \displaystyle\prod\limits_{i = 1}^k ( \frac{\beta^2}{p^2} +\alpha^2 \cdot (1-\frac{2}{p}+\frac{1}{p^2}))\\
	&= \frac{1}{2^k} \cdot ( \frac{\beta^2}{p^2} +\alpha^2 \cdot (1-\frac{2}{p}+\frac{1}{p^2}))^k.
	\end{align*}
	
\subsection{Proof of Proposition~\ref{prop:D_0}}

	First, we assume that $a, a' \in \mathbb{F}_p^{\ell}$ are such that $\Hyp_a = \Hyp_{a'}$, i.e., $a = a'$. Then, by Proposition~\ref{prop:pair_ex} and by our settings of $\alpha$ and $\beta$, we have that
	\begin{align*}
	\Ex_{(z,b) \sim D_0}[D_a(z,b) \cdot D_{a'}(z,b)] &= \frac{1}{2^k} \cdot (\frac{1}{p} \cdot \beta^2 + (1-\frac{1}{p}) \cdot \alpha^2)^k\\
	&= \frac{1}{2^{2k} \cdot p^{(2\ell-1)\cdot k}} \cdot (1+\frac{1}{p-1})^k.
	\end{align*}
	Hence, $D_0[\hat{D}_a \cdot \hat{D}_{a'}] = (p+1-\frac{1}{p-1})^k - 1$, as desired.
	
	Next, we assume that $a, a' \in \mathbb{F}_p^{\ell}$ are such that $\Hyp_a \cap \Hyp_{a'} = \emptyset$. Then, by Proposition~\ref{prop:pair_ex} and by our setting of $\alpha$, we have that
	\begin{align*}
	\Ex_{(z,b) \sim D_0}[D_a(z,b) \cdot D_{a'}(z,b)] &= \frac{1}{2^k}\cdot (\alpha^2 \cdot (1-\frac{2}{p}))^k\\
	&= \frac{1}{2^{3k} \cdot p^{2 k \ell}} \cdot \frac{(1-\frac{2}{p})^k}{(1-\frac{1}{p})^{2k}}.
	\end{align*}
	Hence, $D_0[\hat{D}_a \cdot \hat{D}_{a'}] = \frac{1}{2^k} \cdot \frac{(1-\frac{2}{p})^k}{(1-\frac{1}{p})^{2k}}-1$, as desired.
	
	Finally, we assume that $a, a' \in \mathbb{F}_p^{\ell}$ are such that $\Hyp_a \neq \Hyp_{a'}$ and $\Hyp_a \cap \Hyp_{a'} \neq \emptyset$. Then, by Proposition~\ref{prop:pair_ex} and by our settings of $\alpha$ and $\beta$, we have that
	\begin{align*}
	\Ex_{(z,b) \sim D_0}[D_a(z,b) \cdot D_{a'}(z,b)] &= \frac{1}{2^k} \cdot ( \frac{\beta^2}{p^2} +\alpha^2 \cdot (1-\frac{2}{p}+\frac{1}{p^2}))^k\\
	&= \frac{1}{2^{2k} \cdot p^{2 k \ell}}.
	\end{align*}
	Hence, $D_0[\hat{D}_a \cdot \hat{D}_{a'}] = 0$, as desired.

\end{document}